\documentclass{article}
\usepackage[preprint]{defines/neurips2026/neurips_2026}

\usepackage{silence}
\usepackage{microtype}
\usepackage{graphicx}
\usepackage{subcaption}
\usepackage{caption}
\usepackage{media9}
\usepackage{booktabs}
\usepackage{xspace}

\usepackage{hyperref}
\usepackage{float}

\usepackage{capt-of}

\usepackage{amsmath}
\usepackage{amssymb}
\usepackage{mathtools}
\usepackage{amsthm}

\usepackage[capitalize,noabbrev]{cleveref} 

\theoremstyle{plain}
\newtheorem{theorem}{Theorem}[section]

\theoremstyle{definition}
\newtheorem{definition}[theorem]{Definition}

\theoremstyle{remark}
\newtheorem{remark}[theorem]{Remark}
\usepackage{etoolbox}
\usepackage{listings}
\usepackage{xcolor}

\usepackage[utf8]{inputenc} 
\usepackage[T1]{fontenc}    

\usepackage{amsfonts}       
\usepackage{nicefrac}       

\usepackage[textsize=tiny]{todonotes}

\usepackage{xparse}
\usepackage{etoolbox}
\usepackage{titlesec}

\newcommand{\tokenq}{\textbf{q}}
\newcommand{\tokenk}{\textbf{k}}
\newcommand{\attnscore}{\textbf{s}}
\DeclareRobustCommand{\namedop}[1]{\ifmmode\operatorname{\mathbf{#1}}\else\textup{\textbf{#1}}\fi}
\DeclareRobustCommand{\TIE}{\namedop{TIE}\xspace}

\DeclareRobustCommand{\rope}{\namedop{RoPE}\xspace}
\DeclareRobustCommand{\lmrope}{\namedop{RoTE}\xspace}
\DeclareRobustCommand{\dote}{\namedop{DoTE}\xspace}
\DeclareRobustCommand{\nvedtw}{\ifmmode\mathrm{nV\text{-}EDTW}\else\textup{nV-EDTW}\fi}
\DeclareRobustCommand{\nsedtw}{\ifmmode\mathrm{nS\text{-}EDTW}\else\textup{nS-EDTW}\fi}

\newcommand{\z}{\textbf{z}}

\newcommand{\R}{\textbf{R}}

\newcommand{\A}{\mathbf{A}}
\newcommand{\E}{\mathbb{E}}

\newcommand{\Uniform}{\mathcal U}
\newcommand{\diff}{\text{d}}
\newcommand{\Reals}{\mathbb R}
\newcommand{\omnievent}{\textit{OmniEvents}}

\DeclareMathOperator{\sinc}{sinc}

\newtoggle{isMainText}
\newtoggle{isAppendix}
\newtoggle{isSupplementary}

\lstdefinestyle{pythonstyle}{
    language=Python,
    basicstyle=\ttfamily\small,
    keywordstyle=\color{blue},
    commentstyle=\color{green!60!black},
    stringstyle=\color{orange},
    numbers=left,
    numberstyle=\tiny\color{gray},
    breaklines=true,
    frame=single
}

\colorlet{punct}{red!60!black}
\definecolor{background}{HTML}{EEEEEE}
\definecolor{delim}{RGB}{20,105,176}
\colorlet{numb}{magenta!60!black}

\lstdefinelanguage{jsonformat}{
    basicstyle=\normalfont\ttfamily,
    numbers=left,
    numberstyle=\scriptsize,
    stepnumber=1,
    numbersep=8pt,
    showstringspaces=false,
    breaklines=true,
    frame=lines,
    backgroundcolor=\color{background},
    literate=
     *{0}{{{\color{numb}0}}}{1}
      {1}{{{\color{numb}1}}}{1}
      {2}{{{\color{numb}2}}}{1}
      {3}{{{\color{numb}3}}}{1}
      {4}{{{\color{numb}4}}}{1}
      {5}{{{\color{numb}5}}}{1}
      {6}{{{\color{numb}6}}}{1}
      {7}{{{\color{numb}7}}}{1}
      {8}{{{\color{numb}8}}}{1}
      {9}{{{\color{numb}9}}}{1}
      {:}{{{\color{punct}{:}}}}{1}
      {,}{{{\color{punct}{,}}}}{1}
      {\{}{{{\color{delim}{\{}}}}{1}
      {\}}{{{\color{delim}{\}}}}}{1}
      {[}{{{\color{delim}{[}}}}{1}
      {]}{{{\color{delim}{]}}}}{1},
}

\newcommand{\setVersion}[1]{
    \togglefalse{isMainText}
    \togglefalse{isAppendix}
    \togglefalse{isSupplementary}
    
    \ifstrequal{#1}{maintext}{\toggletrue{isMainText}}{}
    
    \ifstrequal{#1}{appendix}{\toggletrue{isAppendix}}{}
    
    \ifstrequal{#1}{main+appendix}{\toggletrue{isMainText}\toggletrue{isAppendix}}{}
    
    \ifstrequal{#1}{supplementary}{\toggletrue{isSupplementary}}{}
    
    \typeout{*** Compiling version: #1 ***}
}

\titleformat{\paragraph}[runin]
  {\normalfont\normalsize\bfseries}
  {}
  {0pt}
  {}

\titlespacing*{\paragraph}
  {0pt}
  {0.3ex plus 0.1ex minus 0.1ex}
  {0.5em}

\usepackage[bigfiles]{pdfbase}
\ExplSyntaxOn
\NewDocumentCommand\embedvideos{smm}{
  \group_begin:
  \leavevmode
  \tl_if_exist:cTF{file_\file_mdfive_hash:n{#3}}{
    \tl_set_eq:Nc\video{file_\file_mdfive_hash:n{#3}}
  }{
    \IfFileExists{#3}{}{\GenericError{}{File~`#3'~not~found}{}{}}
    \pbs_pdfobj:nnn{}{fstream}{{}{#3}}
    \pbs_pdfobj:nnn{}{dict}{
      /Type/Filespec/F~(#3)/UF~(#3)
      /EF~<</F~\pbs_pdflastobj:>>
    }
    \tl_set:Nx\video{\pbs_pdflastobj:}
    \tl_gset_eq:cN{file_\file_mdfive_hash:n{#3}}\video
  }
  \pbs_pdfobj:nnn{}{dict}{
    /Type/RichMediaInstance/Subtype/Video
    /Asset~\video
    /Params~<</FlashVars (
      source=#3&
      skin=SkinOverAllNoFullNoCaption.swf&
      skinAutoHide=true&
      skinBackgroundColor=0x5F5F5F&
      skinBackgroundAlpha=0
      autoRewind=true
    )>>
  }
  \pbs_pdfobj:nnn{}{dict}{
    /Type/RichMediaConfiguration/Subtype/Video
    /Instances~[\pbs_pdflastobj:]
  }
  \pbs_pdfobj:nnn{}{dict}{
    /Type/RichMediaContent
    /Assets~<<
      /Names~[(#3)~\video]
    >>
    /Configurations~[\pbs_pdflastobj:]
  }
  \tl_set:Nx\rmcontent{\pbs_pdflastobj:}
  \pbs_pdfobj:nnn{}{dict}{
    /Activation~<<
      /Condition/\IfBooleanTF{#1}{PV}{XA}
      /Presentation~<</Style/Embedded>>
    >>
    /Deactivation~<</Condition/PI>>
  }
  \hbox_set:Nn\l_tmpa_box{#2}
  \tl_set:Nx\l_box_wd_tl{\dim_use:N\box_wd:N\l_tmpa_box}
  \tl_set:Nx\l_box_ht_tl{\dim_use:N\box_ht:N\l_tmpa_box}
  \tl_set:Nx\l_box_dp_tl{\dim_use:N\box_dp:N\l_tmpa_box}
  \pbs_pdfxform:nnnnn{1}{1}{}{}{\l_tmpa_box}
  \pbs_pdfannot:nnnn{\l_box_wd_tl}{\l_box_ht_tl}{\l_box_dp_tl}{
    /Subtype/RichMedia
    /BS~<</W~0/S/S>>
    /Contents~(embedded~video~file:#3)
    /NM~(rma:#3)
    /AP~<</N~\pbs_pdflastxform:>>
    /RichMediaSettings~\pbs_pdflastobj:
    /RichMediaContent~\rmcontent
  }
  \phantom{#2}
  \group_end:
}
\ExplSyntaxOff

\newcommand{\embedvideo}[3]{\embedvideos{\includegraphics[width=#3]{#2}}{#1}}

\newcommand{\AuthorList}{}
\newtoggle{firstauthor}
\toggletrue{firstauthor}

\NewDocumentCommand{\Author}{m m O{}}{

    \iftoggle{firstauthor}
        {\togglefalse{firstauthor}}
        {\gappto{\AuthorList}{, }}

    \ifstrempty{#3}
        {\gappto{\AuthorList}{\mbox{#1$^{#2}$}}}
        {\gappto{\AuthorList}{\mbox{#1$^{#2,#3}$}}}
}

\newcommand{\AffiliationList}{}
\newtoggle{firstaffil}
\toggletrue{firstaffil}

\newcommand{\Affiliation}[2]{

    \iftoggle{firstaffil}
        {\togglefalse{firstaffil}}
        {\gappto{\AffiliationList}{, }}

    \gappto{\AffiliationList}{\mbox{$^{#1}$#2}}
}

\newcommand{\SpecialList}{\scriptsize}

\newcommand{\Special}[2]{
    \gappto{\SpecialList}{$^{#1}$#2\quad}
}

\newcommand{\AuthorListWrapper}{\parbox{0.96\textwidth}{\centering\AuthorList}}
\newcommand{\AffiliationListWrapper}{\parbox{0.9\textwidth}{\centering\AffiliationList}}
\newcommand{\SpecialListWrapper}{\parbox{0.75\textwidth}{\centering\SpecialList}}
\title{\TIE: Time Interval Encoding for Video Generation over Events}

\Author{Zhilei Shu}{1,2}[*]
\Author{Shangwen Zhu}{2,3}[*]
\Author{Zihang Liang}{6}[]
\Author{Xiaofan Li}{2}[]
\Author{Qianyu Peng}{8}[]
\Author{Xinyu Cui}{2}[]
\Author{Bo Ye}{2}[]
\Author{Yiming Li}{4}[]
\Author{Fan Cheng}{3}[]
\Author{Jian Zhao}{7}[]
\Author{Yang Cao}{1}[]
\Author{Zheng-Jun Zha}{1}[\dagger]
\Author{Ruili Feng}{5,2}[*]

\Affiliation{1}{University of Science and Technology of China}
\Affiliation{2}{Matrix Team}
\Affiliation{3}{Shanghai Jiao Tong University}
\Affiliation{4}{Nanyang Technological University}
\Affiliation{5}{University of Waterloo}
\Affiliation{6}{The Pennsylvania State University}
\Affiliation{7}{Zhongguancun Academy}
\Affiliation{8}{The University of Hong Kong}

\Special{*}{Equal Contribution}
\Special{\dagger}{Corresponding Author}

\author{
  \AuthorListWrapper 
  \\[1.5em]
  \AffiliationListWrapper 
  \\[0.3em]
  \SpecialListWrapper
  \\[0.2em]
  \url{https://matrixteam-ai.github.io/pages/TIE}
}

\setVersion{main+appendix}
\begin{document}
\maketitle
\iftoggle{isMainText}{
    
\begin{figure*}[h]
    \centering
    \includegraphics[width=\textwidth]{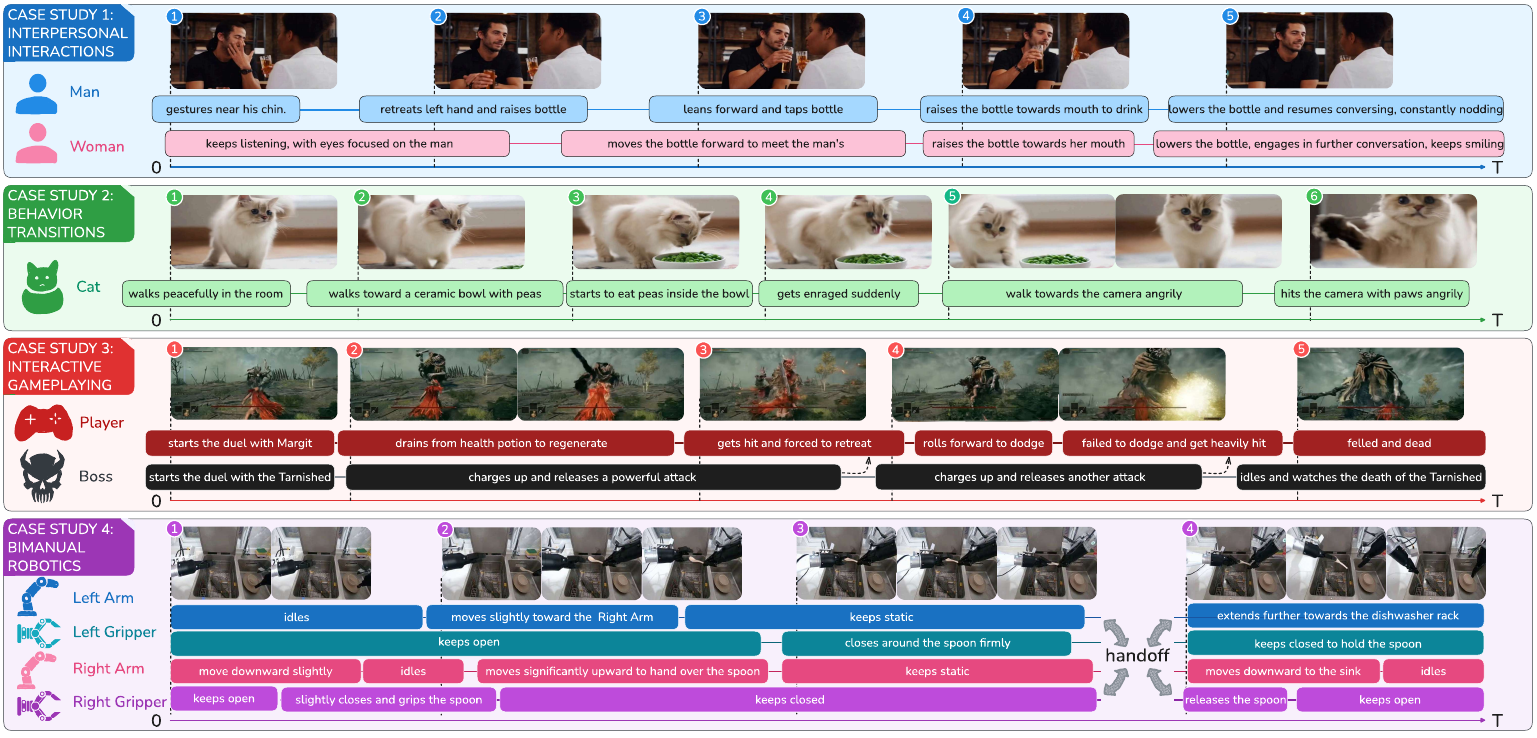}
    \caption{Generation results with \textbf{TIE}.
    By leveraging \textbf{Time Interval Encoding (\TIE)}, the model generates videos from structured event descriptions with explicit temporal boundaries.
    The examples demonstrate accurate event alignment under \textbf{concurrent} and \textbf{overlapping} event settings inaccessible to single-active-prompt baselines, including multi-entity dynamics and interleaved interactions.}
    \label{fig:teaser}
\end{figure*}

\begin{abstract}
Director-style prompting, robotic action prediction, and interactive video agents
demand temporal grounding over concurrent events---a regime in which 68\% of
general clips and over 99\% of robotics/gameplay clips contain overlapping events,
yet existing multi-event generators rest on a single-active-prompt assumption.
However, modern video generators, such as Diffusion Transformers (DiT), represent
time as discrete points through point-wise positional encodings. This formulation
creates a fundamental dimension mismatch: temporally extended intervals and
overlapping events are mathematically unrepresentable to the attention mechanism.
Consequently, event structures collapse into ambiguous token sequences---a
structural limitation that scaling alone does not directly address.
In this paper, we propose \textbf{Time Interval Encoding (\textbf{TIE})}, a
principled, plug-and-play interval-aware generalization of rotary embeddings that
elevates time intervals to first-class primitives inside DiT cross-attention.
Rather than introducing another heuristic interval embedding, we show that, within
RoPE-compatible bilinear attention, \TIE is characterized by two basic principles:
\textit{Temporal Integrability}, which requires an event to aggregate positional
evidence over its full duration, and \textit{Duration Invariance}, which removes
the trivial bias toward longer intervals. Under a uniform kernel, this
characterization yields an efficient closed-form \textit{sinc}-based solution
that preserves the standard attention interface and naturally attenuates boundary
noise through interval integration.
Empirically, \TIE preserves the visual quality of the base DiT model while
substantially improving temporal controllability. In our experiments on the
\omnievent~dataset, it improves human-verified Temporal Constraint Satisfaction
Rate from 77.34\% to 96.03\% and reduces temporal boundary error from 0.261s to
0.073s, while also improving trajectory-level temporal alignment metrics. \TIE
also remains effective under noisy timestamp perturbations, a key requirement for
large-scale event-conditioned training where interval annotations are often
automatic and imperfect, and reaches comparable visual and temporal-control
quality in fewer training steps. The code and dataset are available at \url{https://github.com/MatrixTeam-AI/TIE}.
\end{abstract}
    \section{Introduction}
The landscape of video generation has undergone a paradigm shift, with Diffusion Transformers (DiT) pushing the boundaries of visual fidelity \cite{peebles2023scalable,wan2025,lu2024vdt,ma2025latte,yang2025cogvideox,singer2023makeavideo,blattmann2023alignlatents} and creative customizability \cite{chen2023controlavideo, vace,wang2023videocomposer, guo2024animatediff, yang2024directavideo}. Beyond artistic content creation, these models are increasingly envisioned as \textbf{world simulators} \cite{hu2023gaia1} for robotics \cite{nvidia2025cosmosworldfoundationmodel} and interactive agents \cite{yang2024unisim,bruce2024genie,du2023unipi,valevski2024gamengen}. In robotics research, video-generation-based data synthesis \cite{zhou2024robodreamer} has emerged as a powerful tool for policy pre-training and data enrichment \cite{huang2025enerverse,zhu2025irasim}. However, the transition from ``visually pleasing'' to ``functionally useful'' videos requires a level of temporal precision that remains elusive. For applications such as bimanual robot manipulation or multi-athlete sports analysis \cite{li2021multisports}, a model must strictly adhere to the \textit{precise temporal boundaries}, ordering, and overlapping of multiple interacting events.

The fundamental bottleneck in achieving such precision is a \textbf{representation deficit} in current temporal positional encodings. While videos are continuous spatiotemporal signals, modern DiTs predominantly treat time as a sequence of discrete \textit{points}. Standard rotary positional encodings (\textbf{RoPE}), originally designed for text, are point-wise by nature \cite{su2024roformer}. This formulation creates a structural dimension mismatch: while real-world events (e.g., ``a robot arm gripping a tool'') occupy continuous \textbf{time intervals}, the model processes them as if they were tied to infinitesimal timestamps. Consequently, event structures collapse into ambiguous token sequences, leading to temporal drifting and the failure of concurrent event coordination---a limitation that cannot be resolved by scaling model size alone. 
Moreover, in large-scale real-world video corpora, event intervals are rarely available with frame-perfect boundaries: they are often produced by VLMs, action detectors, or other automatic annotators, whose start and end times are inevitably noisy. 
Thus, a useful interval representation must not only encode temporal support, but also remain stable under imperfect boundary annotations.

Existing works on timeline-guided synthesis \cite{villegas2022phenaki,oh2024mevg,wu2025mint,lin2024videodirectorgpt,qiu2024freenoise,cai2025ditctrl,yin2023nuwaxl} often assume a ``single-active-prompt'' constraint, where only one text prompt is valid at any given timestamp. This is insufficient for complex physical interactions where events are frequently \textbf{interleaved or staggered}. For instance, in a bimanual assembly task, the ``reaching'' phase of one arm must be precisely coordinated with the ``holding'' phase of the other. To bridge this gap, video generation models require an encoding mechanism that elevates time intervals to first-class primitives, allowing for a principled handling of event duration and concurrency.

This paper introduces \textbf{Time Interval Encoding (TIE)}, a novel interval-aware formulation that explicitly models event durations and boundaries within the cross-attention mechanism. Our starting point is two basic principles: \textit{Temporal Integrability}, which requires an event’s contribution to aggregate positional evidence over its full duration, and \textit{Duration Invariance}, which prevents matching strength from growing trivially with interval length. We show that, within RoPE-compatible bilinear attention, these two principles characterize \TIE rather than merely motivate it. Under the uniform kernel, the two principles lead to a closed-form sinc-modulated encoder (\lmrope) that incurs zero runtime overhead over standard \textbf{RoPE} and reduces to \textbf{RoPE} in the point-wise limit $r\rightarrow0$. 
The same interval-integrated form also acts as a temporal low-pass filter: boundary perturbations affect only the marginal portion of the event support, making \lmrope naturally robust to noisy start and end times.

To validate \TIE in practice, we construct \omnievent, a structured event-prompt dataset consisting of three parts: 1). \textit{PexelsEvents} with 250k clips of general case video, 2). \textit{RoboticsEvents} with 86k clips of task-specific robotics, and 3). \textit{GameEvents} with 80k clips of gameplay videos collected from EldenRing. We integrate \TIE into Wan2.2-5B-TI2V and evaluate it on \omnievent. Our experiments are organized around four questions: whether \TIE remains compatible with existing DiT generators, whether it improves time interval controllability, whether it remains robust under noisy interval boundaries for large-scale real-world training, and whether the gain truly comes from interval-aware encoding. Empirically, \TIE preserves standard visual quality, substantially improves human-verified temporal grounding (raising TCSR from 77.34\% to 96.03\% and reducing temporal boundary error from 0.261s to 0.073s), improves trajectory-level temporal alignment metrics, remains effective under noisy timestamp perturbations, and reaches comparable visual and temporal-control quality in fewer training steps. Qualitative case studies on \textit{EldenRing} further show localized editing of future interactions without corrupting earlier events. (Video examples can be found in \Cref{fig:vid_preview})

Our contributions are summarized as follows:
\begin{itemize}
    \item We propose Time Interval Encoding (\TIE) for the concurrent-event regime that single-active-prompt methods cannot represent. Two principles—Temporal Integrability and Duration Invariance—uniquely determine a closed-form sinc-modulated encoder (\lmrope) with zero overhead over \textbf{RoPE}.
    \item We establish a structured-prompt evaluation setting for interval-conditioned video generation, including the \omnievent~dataset and human-verified temporal-grounding metrics such as TCSR and trajectory-level metrics.
    \item We show that \TIE is robust to noisy temporal boundaries, a core requirement for large-scale event-conditioned pretraining where interval annotations are often produced by imperfect automatic annotators.
    \item We show empirically that \TIE preserves base-model visual quality, substantially improves human-verified and trajectory-level temporal grounding, accelerates the convergence of temporal grounding, and supports controllable multi-subject interactions.
\end{itemize}

\begin{figure}
    \centering
    \embedvideo{assets/vids/example1.mp4}{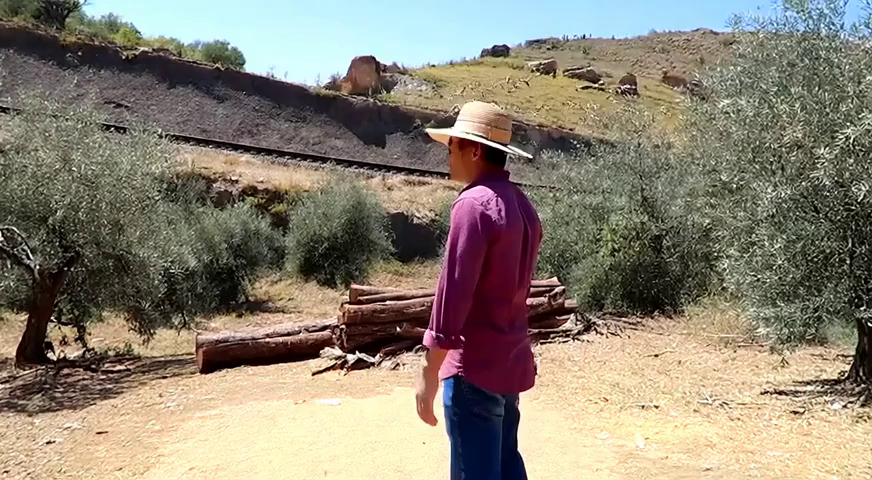}{0.19\linewidth}
    \embedvideo{assets/vids/example2.mp4}{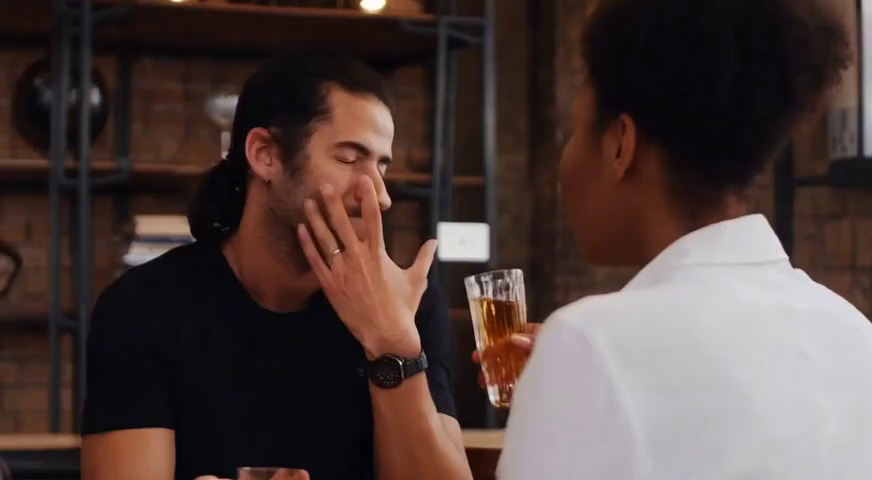}{0.19\linewidth}
    \embedvideo{assets/vids/long_robotic_task.mp4}{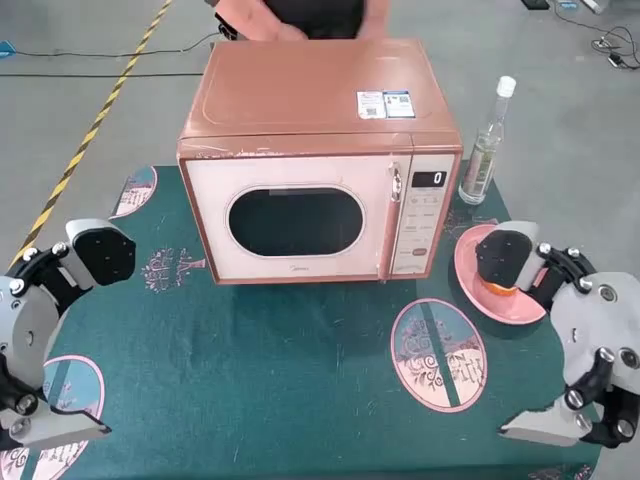}{0.13933\linewidth}
    \embedvideo{assets/vids/example4.mp4}{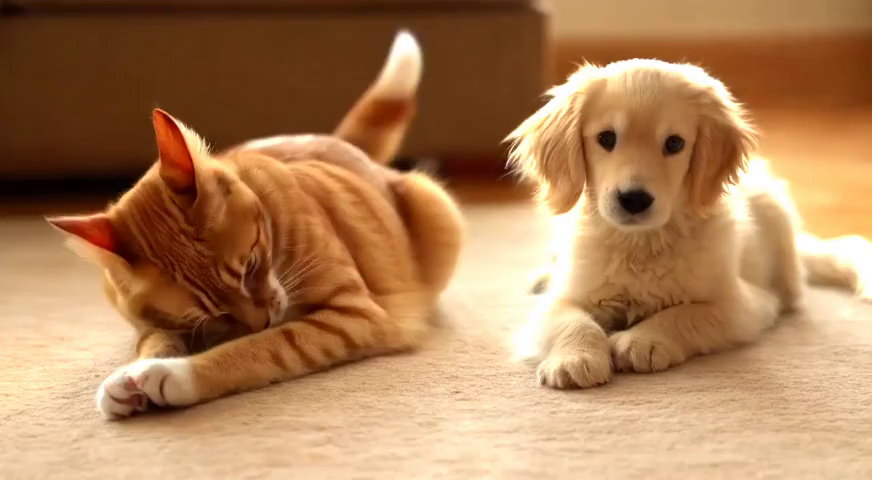}{0.19\linewidth}
    \embedvideo{assets/vids/example5.mp4}{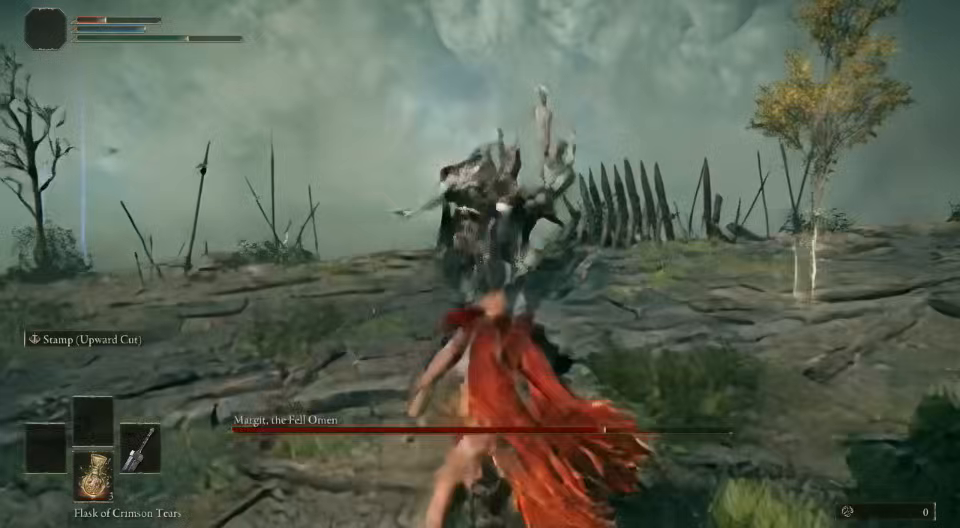}{0.19\linewidth}
    \vspace{-6pt}
    \caption{Video examples: \textit{Best viewed with Acrobat
Reader (Also works on Foxit PDF Reader / PDF-XChange Editor). Click the images to play the animation clips.}}
    \label{fig:vid_preview}
    \vspace{-10pt}
\end{figure}

    \section{Related Work}

Multi-event video generation conditions the synthesis on a sequence of event-specific prompts. Across paradigms, we observe that all existing methods share a single structural assumption — the single-active-prompt constraint: at most one event prompt is valid at any given timestamp. Sequential-token approaches such as Phenaki~\cite{villegas2022phenaki} translate disjoint prompts into a continuous token stream; diffusion-based extenders Gen-L-Video~\cite{wang2023gen} and MEVG~\cite{oh2024mevg} denoise non-overlapping local windows; multi-agent systems DreamFactory~\cite{xie2024dreamfactory} and VGoT~\cite{zheng2024videogen} structure transitions between sequentially scheduled events. Recent temporal-encoding methods inherit the same constraint: MinT~\cite{wu2025mint} aligns multi-event features through positional embeddings indexed at scalar timestamps, TS-Attn~\cite{zhang2026tsattntemporalwiseseparableattention} applies separable mask-based attention over disjoint prompt-time slots, and LongLive~\cite{yang2025longlive} preserves single-prompt KV state via KV-recache. In all of the above, each event is bound to a temporal point or a disjoint slot, so concurrent and overlapping events are inexpressible by construction. Yet physically realistic interactions—bimanual manipulation, multi-subject sports, adversarial combat—are inherently concurrent: in \omnievent, two or more events overlap in 68\% of general clips and over 99\% of robotics and gameplay clips. Closing this gap requires elevating the time interval itself to a first-class primitive, rather than retrofitting concurrency onto sequential schedules.

    \section{Method}\label{sec:method}

\begin{figure*}[ht]
    \centering
    \includegraphics[width=\textwidth]{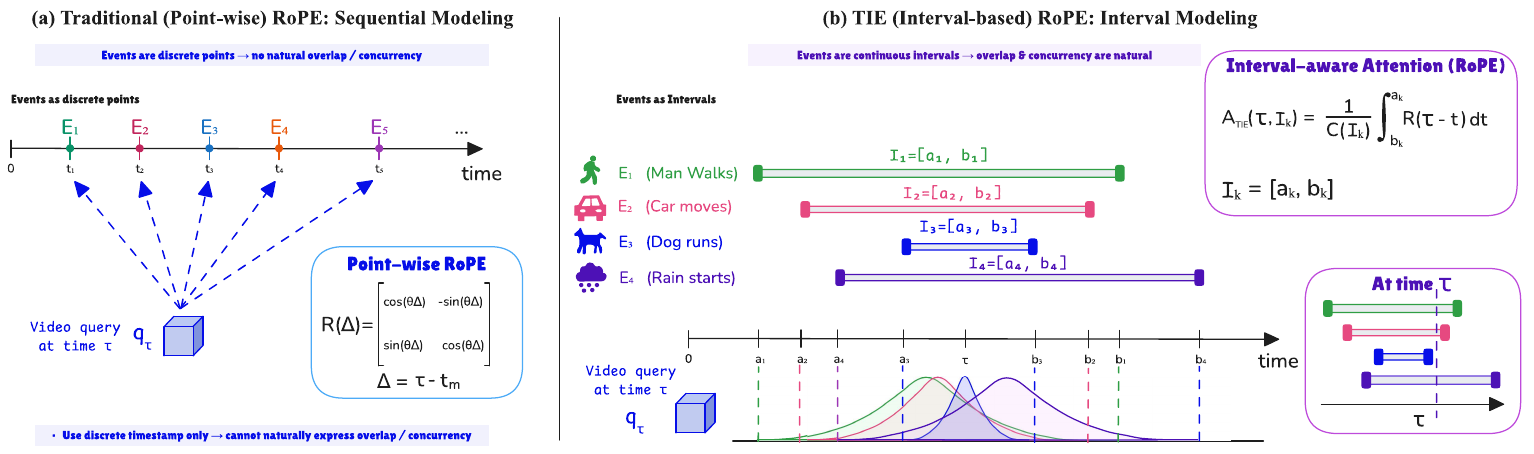}
    \caption{Comparison between normal RoPE (Point-wise) and \TIE~(Interval-based): The point-wise RoPE cannot naturally model the point-to-interval event activation. While our \TIE supports natural modeling of time interval.}
    \label{fig:tie_visualize}
\end{figure*}

\subsection{Time Interval Encoding}\label{sec: RoTE}
To bridge point-wise video features and interval-based textual events, we propose \textbf{Time Interval Encoding (\textbf{TIE})}. We treat a video token as a temporal point with midpoint timestamp $m_i\in\Reals$, and a textual event token as a time interval $I_j=[t_j^s,t_j^e]$ (find visualization explanation in \ref{fig:tie_visualize}). The goal is to extend \rope from point-to-point matching to point-to-interval matching while preserving the standard dot-product attention interface.

\paragraph{\rope Review.}
For a token $\z\in\Reals^d$ at time $t$, standard \rope applies $\mathrm{RoPE}(\z,t)=\R_{t}\z$,
where
\begin{equation}
    \R_{t}=\operatorname{diag}(\A_{1,t},\A_{2,t},\dots,\A_{{d/2},t}),
    \qquad
    \A_{i,t}=
    \begin{bmatrix}
        \cos(\theta_i t) & -\sin(\theta_i t) \\
        \sin(\theta_i t) & \cos(\theta_i t)
    \end{bmatrix}.
\end{equation}
Its key property is
\begin{equation}
    s_{\mathrm{rope}}(\tokenq,\tokenk;t_q,t_k)
    =
    \mathrm{RoPE}(\tokenq,t_q)^\top\mathrm{RoPE}(\tokenk,t_k)
    =
    \tokenq^\top\R_{t_k-t_q}\tokenk.
    \label{eqn:rope_relative}
\end{equation}
\rope therefore models relations between temporal \emph{points}; \TIE extends it to temporal \emph{intervals}.

\paragraph{Core \TIE Formulation.}
To define a principled point-to-interval score, the governing requirements should be both natural and necessary. Since an event is temporally extended, a reasonable score should aggregate evidence over the full interval rather than collapse it to a single timestamp; otherwise interior evidence is discarded and temporal aliasing is introduced. At the same time, the score should remain comparable across intervals of different lengths; otherwise the accumulated similarity grows trivially with duration and systematically biases attention toward long events. These considerations naturally lead to the following two principles.

\begin{definition}[Two Principles of \TIE]\label{def:tie_principles_main}
An interval-aware score satisfies the \TIE principles if:
\smallskip
\begin{itemize}
    \item (\textit{Temporal Integrability.}) The raw interval score is defined by integrating point-wise \rope evidence over the full temporal support of the event.
Formally, for some probability kernel $\mu_I$ supported on $I$ with $\mu_I(I)>0$,
\begin{equation*}
    \bar{s}(\tokenq,\tokenk;t_q,I)
    :=
    \mathbb{E}_{\tau\sim \mu_I} [s_{\mathrm{rope}}(\tokenq,\tokenk;t_q,\tau)].
\end{equation*}
\item (\textit{Duration Invariance.}) The final interval-aware score should reflect semantic relevance rather than grow trivially with duration.
Accordingly,
\begin{equation*}
    s(\tokenq,\tokenk;t_q,I)
    :=
    \frac{1}{C(\mu_I)}\bar{s}(\tokenq,\tokenk;t_q,I),
    \qquad C(\mu_I)>0,
\end{equation*}
where $C(\mu_I)$ depends only on $\mu_I$ and is independent of $\tokenq$, $\tokenk$, and $t_q$.
\end{itemize}
\end{definition}

These principles are not merely intuitive desiderata. \emph{Temporal Integrability} is necessary because an event is a temporally extended object: collapsing $I$ to its center or boundaries discards interior evidence and introduces temporal aliasing. \emph{Duration Invariance} is equally necessary because, without normalization, the accumulated similarity grows trivially with interval length and systematically biases attention toward long events. In this sense, the principles are both reasonable and indispensable.

The two principles can characterize \textbf{TIE}. First, by \emph{Temporal Integrability},
\begin{equation}
    \bar s = \mathbb{E}_{\tau\sim \mu_I} [s_{\mathrm{rope}}(\tokenq,\tokenk;t_q,\tau)] = \mathrm{RoPE}(\tokenq,t_q)^\top\mathbb{E}_{\tau\sim \mu_I}[\mathrm{RoPE}(\tokenk;\tau)].
\end{equation}
Therefore, \TIE must be the form:
\[\mathrm{TIE}(k,I)\propto \mathbb{E}_{\tau\sim \mu_I}[\mathrm{RoPE}(\tokenk;\tau)].\]
By duration invariance, we require the renormalized coefficient should satisfy
\[1 = \mathbb{E}_{\tokenq\sim \Uniform(S^{d-1})}[\mathrm{RoPE}(\tokenq,c)^\top\mathrm{TIE}(\tokenq,I)],\]
where $c$ denotes the midpoint of interval $I$. Solving the equation above shows that the coefficient does not depend on a specific $\tokenq$:
\begin{equation}
\mathrm{TIE}(\tokenk,I) = \frac{1}{C(\mu_I)}\mathbb{E}_{\tau\sim \mu_I}[\mathrm{RoPE}(\tokenk,\tau)].
    \label{eqn:tie_single_kernel}
\end{equation}
The corresponding score will be
\begin{equation}
s(\tokenq,\tokenk,t_q,I):= \mathrm{RoPE}(\tokenq,t_q)^\top\mathrm{TIE}(\tokenk,I). 
\label{eqn:tie_score}
\end{equation}
\subsection{A Closed-Form Instantiation: Uniform Kernel Yields Sinc-Modulated RoPE}

The characterization theorem above is constructive: once the probability kernel is fixed, the corresponding \TIE instantiation is fixed as well. In particular, choosing the uniform kernel does not merely suggest a convenient variant; it yields a closed-form encoder compatible with the two principles under uniform averaging.

\begin{theorem}[Uniform-kernel \TIE yields the closed-form $\lmrope$]\label{cor:rote_uniform_main}
Given $I = [t^s,t^e]$, let $\mu_I=\Uniform([t^s,t^e])$, and define the interval center and radius by $c=(t^s+t^e)/2$ and $r=(t^e-t^s)/2$. Then the corresponding \TIE encoder has the closed form
\begin{equation}
    {\A}_{i,c,r}
    =
    \sinc(\theta_i r)
    \begin{bmatrix}
        \cos(\theta_i c) & -\sin(\theta_i c) \\
        \sin(\theta_i c) & \cos(\theta_i c)
    \end{bmatrix}.
\end{equation}
Stacking the blocks gives
    ${\R}_{c,r}
    =
    \operatorname{diag}({\A}_{1,c,r},{\A}_{2,c,r},\dots,{\A}_{{d/2},c,r})$,
and the resulting interval encoder is
    $\mathrm{RoTE}(\tokenk,c,r)
    =
    \frac{1}{ {C}_{r}}{\R}_{c,r}\tokenk$,
where $ C_{r}=\frac{2}{d}\sum_{\ell=1}^{d/2}\sinc(\theta_\ell r).$
Consequently, for a query token $\tokenq_i$ at time $m_i$ and a key token $\tokenk_j$ associated with interval $I_j=[t_j^s,t_j^e]$, the attention score becomes
\begin{equation}
    \attnscore_{i,j}
    \propto
    \mathrm{RoPE}(\tokenq_i,m_i)^\top\mathrm{RoTE}(\tokenk_j,c_j,r_j),
    \qquad
    c_j=\frac{t_j^s+t_j^e}{2},\;
    r_j=\frac{t_j^e-t_j^s}{2}.
    \label{eqn:rote_score}
\end{equation}
\end{theorem}
The normalization constant is chosen to preserve unit expected self-overlap at the interval center, preventing attention scores from depending trivially on interval length. The center $c$ controls the rotation phase, while the radius $r$ suppresses high-frequency bands through $\sinc(\theta r)$. Long intervals therefore act as a temporal low-pass filter, and $\lmrope$ reduces to standard \rope as $r\to0$. This gives the promised one-to-one relationship: the two principles characterize the \TIE family, and the uniform-kernel choice singles out the sinc-modulated closed-form solution.

\paragraph{Boundary-only Ablation.}
For ablations we also consider the Dirac-kernel variant \dote, which keeps only the interval boundaries:
    $\mathrm{DoTE}(\tokenk,I)
    =
    \mathbf{Concat}(\mathrm{RoPE}(\tokenk^{(s)},t^s),\mathrm{RoPE}(\tokenk^{(e)},t^e)).$
Unlike $\lmrope$, this variant ignores the interval interior. Its formal derivation is also deferred to \Cref{append:proofs_tie}.
In practice, we scale every timestamp with a scaling factor $\gamma$, that is, $c = \gamma\frac{t^s+t^e}2, r=\gamma\frac{t^e-t^s}2$, so that we can better control the decaying speed. By default we choose $\gamma=4.0$.

\subsection{\lmrope for Large-Scale Pretraining}
\label{sec:lmrope_large_scale_pretraining}

A central requirement for large-scale event-conditioned video pretraining is tolerance to imperfect temporal boundaries. In small curated benchmarks, event intervals may be manually verified with high precision. At scale, however, event-time annotations are more likely to come from VLMs, action detectors, captioning models, or other automatic perception systems. These annotators can provide useful event semantics, but their predicted start and end times are inevitably noisy. Therefore, an interval-aware temporal encoding should not depend on frame-perfect endpoints. It should be able to treat noisy intervals as coarse temporal support for event grounding. The following theorem formalizes this property for $\lmrope$. The key distinction is that $\rope$ and $\dote$ encode events through pointwise temporal phases: $\rope$ collapses an event to a single timestamp, while $\dote$ places all temporal information on the two boundaries. Consequently, perturbing the timestamp or either endpoint directly perturbs the encoded phase. In contrast, $\lmrope$ represents an event by a duration-normalized integral of $\rope$ features over the interval. Boundary perturbations therefore affect only the marginal portion of the integration domain, yielding a robustness behavior controlled by the relative boundary disturbance.
\begin{theorem}[Expected robustness of \lmrope to timestamp noise]
\label{thm:rote_noise_robustness}
Define the unscaled interval-attention score
\begin{equation}
    \attnscore_{i,j}(c,r)
    :=
    \mathrm{RoPE}(\tokenq_i,m_i)^\top \mathrm{RoTE}(\tokenk_j,c,r),
    \qquad
    \mathrm{RoTE}(\tokenk,c,r)=C_r^{-1}\mathbf{R}_{c,r}\tokenk .
\end{equation}
For event $I_j=[t_j^s,t_j^e]$, let $c_j=(t_j^s+t_j^e)/2$ and
$r_j=(t_j^e-t_j^s)/2>0$. Suppose its observed endpoints are
\begin{equation}
    \widetilde t_j^s=t_j^s+\epsilon_j^s,
    \qquad
    \widetilde t_j^e=t_j^e+\epsilon_j^e,
    \qquad
    |\epsilon_j^s|,|\epsilon_j^e|\le \delta<r_j
    \quad\text{a.s.}
\end{equation}
Write
\begin{equation}
    \Delta c_j:=\frac{\epsilon_j^s+\epsilon_j^e}{2},
    \qquad
    \Delta r_j:=\frac{\epsilon_j^e-\epsilon_j^s}{2},
    \qquad
    \widetilde c_j=c_j+\Delta c_j,
    \qquad
    \widetilde r_j=r_j+\Delta r_j .
\end{equation}
Assume the normalization does not degenerate on the perturbed radius range,
\begin{equation}
    C_{\min,j}:=
    \inf_{\rho\in[r_j-\delta,r_j+\delta]} |C_\rho|>0 .
\end{equation}
Then
\begin{equation}
\left|
\attnscore_{i,j}(\widetilde c_j,\widetilde r_j)
-
\attnscore_{i,j}(c_j,r_j)
\right|
\le
\frac{\left\lVert\tokenq_i\right\rVert\left\lVert\tokenk_j\right\rVert\delta}{r_j-\delta}
\left(
\frac{3}{C_{\min,j}}+
\frac{2}{C_{\min,j}^2}
\right).
\end{equation}
Therefore, if $\delta\le r_j/2$ and $C_{\min,j}$ is bounded away from zero, \lmrope has expected
noise sensitivity $O(\left\lVert\tokenq_i\right\rVert\left\lVert\tokenk_j\right\rVert\delta/r_j)$. By contrast,
point-wise \rope and boundary-only \dote have local worst-case timestamp sensitivity
$O(\left\lVert\tokenq_i\right\rVert\left\lVert\tokenk_j\right\rVert\theta_{\max}\delta)$, where
$\theta_{\max}:=\max_\ell\theta_\ell$, with no decay in the interval radius $r_j$.
\end{theorem}
\begin{remark}[Interpretation of the robustness bound]
The theorem shows that the timestamp-noise sensitivity of \lmrope is controlled
by the relative boundary error $\delta/r_j$ when the normalization is
non-degenerate. This relative scaling is the key difference from point-wise or
boundary-only encodings. Since \lmrope aggregates positional evidence over the
full temporal support of the event, perturbing the start or end time only changes
the marginal portion of the interval representation. As a result, the effect of
boundary noise is averaged over the event duration, and long intervals become
increasingly insensitive to small endpoint errors. In contrast, point-wise
\rope collapses the event to a single timestamp, while boundary-only \dote
places all temporal information on the two endpoints; in both cases, timestamp
noise directly perturbs the encoded phase and does not decay with the interval
radius.
\end{remark}

\begin{remark}[Low-pass smoothing from the closed-form interval integral]
The robustness of $\lmrope$ can also be understood from the closed-form
frequency response of the interval integral. For each RoPE frequency $\theta_\ell$,
the uniform-kernel interval encoder multiplies the corresponding rotation block
by $\mathrm{sinc}(\theta_\ell r)$. Therefore, perturbations of the interval center
are attenuated at high temporal frequencies:
\begin{equation}
    \theta_\ell \left|\mathrm{sinc}(\theta_\ell r)\right|
    =
    \left|\frac{\sin(\theta_\ell r)}{r}\right|
    \le \frac{1}{r}.
\end{equation}
Thus, up to the non-degenerate normalization factor, the local sensitivity of
$\lmrope$ to center perturbations decays with the interval radius. For small
intervals, this behavior reduces to the point-wise \rope regime; for long
intervals, the interval integral acts as a temporal low-pass filter and suppresses
high-frequency phase noise. This smoothing effect is absent in point-wise
$\rope$ and boundary-only $\dote$, whose temporal information is concentrated on
one timestamp or two endpoints and therefore does not become less sensitive as
the interval becomes longer.
\end{remark}

\begin{remark}[Why learned interval embeddings do not provide the same guarantee]
A learned interval embedding could in principle become robust to noisy boundaries
if it is trained with sufficiently diverse boundary perturbations. However, this
robustness would be distribution-dependent: it holds only for the noise patterns
seen during training and provides no structural guarantee that longer intervals
should be less sensitive to endpoint errors. In contrast, the robustness of
$\lmrope$ is architectural. The $\mathrm{sinc}$ attenuation is built into the
closed-form interval representation itself, so the smoothing effect does not rely
on learning a particular noise distribution or adding explicit boundary-noise
augmentation. This property is especially important for large-scale pretraining,
where event intervals may be generated by different VLMs or detectors whose
boundary errors vary across domains.
\end{remark}

\begin{remark}[Implication for large-scale event-conditioned pretraining]
This robustness is important for scaling interval-aware video generation beyond
small curated datasets. In real-world data collection, event intervals are often
obtained from VLMs, action detectors, captioning models, or other neural
annotators. Such pipelines can identify the correct event semantics at scale, but
their start and end timestamps are inevitably approximate. Therefore, robustness
to boundary noise is not merely a desirable property of \lmrope; it is a core
requirement for using interval-conditioned objectives in large-scale pretraining.
The $\delta/r_j$ sensitivity indicates that \lmrope can use noisy intervals as
coarse temporal support, rather than depending on frame-perfect endpoint
supervision.
\end{remark}

    \section{Experiments}
\label{sec:experiments}


We evaluate \TIE as a plug-and-play interval-aware encoding for DiT-based video generation. \omnievent~is deliberately constructed in the concurrent-event regime (68\%–99\% per-clip overlap probability), where sequential multi-event methods such as MinT~\cite{wu2025mint}, TS-Attn~\cite{zhang2026tsattntemporalwiseseparableattention}, MEVG~\cite{oh2024mevg} are structurally inapplicable—their single-active-prompt assumption disallows overlapping prompts. Our experiments therefore center on three claims: (i) \TIE preserves the visual prior of the base DiT, (ii) \TIE substantially improves interval-level controllability under concurrent prompts via human-verified and trajectory-level metrics, (iii) \TIE remains robust under noisy interval boundaries, a key requirement for large-scale real-world training where event-time annotations are typically obtained from imperfect VLM- or detector-based pipelines, and, and (iv) ablations verify that gains arise from interval integration rather than boundary-only or timestamp encoding. We retain comparisons against MinT and TS-Attn on their native sequential benchmarks (StoryEval, StoryBench) in~\Cref{sec:exp_time_precise} and~\Cref{append:moreresults}. We use \lmrope—the uniform-kernel closed-form instance of \TIE—as the default.

\subsection{Experimental Setup}

\paragraph{Backbone.} We use Wan2.2-5B-TI2V as the base video generator, following its standard DiT architecture and training recipe \cite{wan2025}. \TIE is inserted only into the text-to-video cross-attention pathway by replacing point-wise temporal \rope on event tokens with interval-aware key encodings.


\paragraph{Datasets.} 
We construct the \omnievent~ dataset, which consists of three complementary domains. For general event-conditioned video generation, we collect 250k clips from Pexels-400k \cite{pexels400k}, denoted as \textit{PexelsEvents}. For high-precision action-centric evaluation, we collect 80k clips of EldenRing gameplay videos with frame-accurate event intervals recorded from game memory, denoted as \textit{GameEvents}. For task-specific domain applications, we annotate 86k clips from the Agibot-World-Alpha \cite{contributors2024agibotworldrepo}, forming the \textit{RoboticsEvents} part. Each training sample is represented as a set of event tokens, where each event is associated with a textual description and a time interval (\Cref{fig:textual_timestamps} shows an example case). Details about the \omnievent~can be found in Appendix \ref{append:data}.

\paragraph{Baselines.} We compare against several variants: \textbf{Base}, the original Wan2.2 model without event-interval fine-tuning; \textbf{Finetuned}, a standard event-conditioned fine-tuning baseline that concatenates event descriptions and timestamps into the prompt but keeps point-wise temporal encoding; \textbf{DoTE}, a boundary-only variant that encodes only the start and end timestamps; and \textbf{RoTE}, our main method. Dense mask-aware attention is not used as a primary baseline because frame–event pairwise masks or biases break FlashAttention compatibility, whereas \TIE preserves the base DiT’s FlashAttention path by keeping the standard QK attention form. We do not introduce an additional learned time-embedding baseline because it would primarily test another parameterization of timestamp conditioning; instead, our ablation directly isolates the mechanism of interest: implicit timestamp learning and full interval integration.

\paragraph{Evaluation metrics.} 
We evaluate three aspects. First, for visual quality and semantic alignment, we report FID, FVD, and VideoScore dimensions including visual quality, temporal consistency, dynamic degree, text alignment, and factual consistency \cite{Heusel2017GANsTB,Unterthiner2018TowardsAG,he2024videoscore}. These metrics test whether \TIE preserves the base model's visual generation ability. Second, for time interval controllability, we use both human-verified event-level metrics and temporal alignment metrics. Our primary human metric is Temporal Constraint Satisfaction Rate (TCSR), which summarizes the total temporal alignment degree, together with breakdowns into Event Occurrence, Temporal Error, Order Accuracy and Overlap Accuracy. We report CLIP-Event, EMD and nDTW as trajectory-level temporal alignment proxies. CLIP-Event calculates the per-event CLIP Score for generated videos at frame-level, while EMD and nDTW compare the generated videos with reference ones in temporal alignment. These metrics complement TCSR by measuring whether the generated visual or semantic trajectory follows the target temporal structure.

\subsection{\TIE Preserves Visual Quality and DiT Compatibility}

A practical temporal-control module should improve controllability without damaging the visual prior of a pretrained DiT generator. We first evaluate whether \TIE preserves video generation quality.
\Cref{tab:expv3_video_quality} reports FID, FVD, and VideoScore results on \textit{PexelsEvents}. Compared with the Finetuned baseline, \TIE achieves comparable or better FID/FVD while maintaining similar visual quality, temporal consistency, text alignment, and factual consistency. This indicates that interval-aware event encoding does not disrupt the base model's visual generation distribution. Instead, \TIE improves temporal grounding while remaining compatible with the original DiT backbone.
These results support our first claim: \TIE is a plug-and-play interval encoding module that preserves the visual and semantic behavior of the underlying DiT generator.\vspace{-0.2cm}
\begin{table*}[htb]
\centering
\small
\setlength{\tabcolsep}{4pt}
\begin{tabular}{c|cc|ccccc|ccccc}
\hline
& \multicolumn{7}{c|}{Text-to-Video (T2V)} & \multicolumn{5}{c}{Image-to-Video (I2V)} \\
\cline{2-13}
Method 
& FID$\downarrow$ & FVD$\downarrow$ 
& VQ$\uparrow$ & TC$\uparrow$ & DD$\uparrow$ & TA$\uparrow$ & FC$\uparrow$
& VQ$\uparrow$ & TC$\uparrow$ & DD$\uparrow$ & TA$\uparrow$ & FC$\uparrow$ \\
\hline

Base 
& 59.68 & 357.51 
& 2.73 & 2.78 & 2.58 & 2.68 & 2.65
& 2.98 & 2.94 & 2.93 & 2.87 & 2.88 \\

Finetuned 
& 43.74 & 234.40 
& 3.03 & 3.01 & 2.97 & 2.86 & 2.94
& 3.26 & 3.21 & 3.28 & 3.06 & 3.12 \\

\TIE (\lmrope) 
& \textbf{42.53} & \textbf{217.29} 
& \textbf{3.10} & \textbf{3.05} & \textbf{3.06} & \textbf{2.92} & \textbf{2.99}
& \textbf{3.30} & \textbf{3.22} & 3.27 & \textbf{3.07} & \textbf{3.18} \\

\hline
\end{tabular}

\caption{
Quality on PexelsEvents (68\% per-event overlap; sequential-event methods inapplicable, see~\Cref{sec:exp_time_precise} / ~\Cref{tab:expv3_storybench}). I2V uses the first frame as visual condition. VQ/TC/DD/TA/FC = Visual/Temporal/Dynamic/Text/Factual \{Quality, Consistency, Degree, Alignment, Consistency\}.
}
\label{tab:expv3_video_quality}
\end{table*}\vspace{-0.3cm}



\subsection{\TIE Improves Time Interval Controllability}\label{sec:exp_time_precise}

We next evaluate the central claim of this work: whether \TIE improves the model's ability to generate events at specified time intervals, including sequential, overlapping, and concurrent event structures.

\paragraph{Human-Verified Temporal Constraint Satisfaction.}
Standard metrics such as FVD, CLIP, and VideoScore are insufficient for precise temporal grounding: they may judge two videos similarly even when an action starts several seconds too early or too late. We therefore introduce a structured human verification protocol that directly checks whether the generated video satisfies the temporal constraints specified in the event prompt.
For each generated video, annotators are shown the video and the structured event prompt. For each requested event, they verify whether the event occurs and, if it occurs, record the start- and end-time deviations from the requested interval $[t_s,t_e]$. Higher-level metrics such as order accuracy, overlap accuracy, and TCSR are then computed from these event-level occurrence and boundary annotations. Annotators do not see the model identity and do not provide open-ended preference judgments.
We sample 100 prompts and generate 100 videos with the finetuned baseline and 100 videos with \textbf{TIE}. For each event, 10 human annotators verify event occurrence and record the start- and end-time deviations from the requested interval. All metrics are computed from the aggregated human annotations. We report five event-level metrics: \textbf{Event Occurrence} measures event realization; \textbf{Temporal Error} measures boundary deviation with missing-event penalties; \textbf{Order Accuracy} measures before/after consistency; \textbf{Overlap Accuracy} measures concurrency preservation; and \textbf{TCSR} measures overall \textbf{T}emporal \textbf{C}onstraint \textbf{S}atisfaction \textbf{R}ates.
TCSR is computed as the fraction of satisfied event-level and relation-level constraints. An interval constraint is satisfied only if the event occurs and both start/end deviations are within a 0.25s tolerance; order and overlap constraints require the corresponding events to occur and preserve the requested temporal relation. See \Cref{sec:human_eval} for the full annotation protocol and metric definitions.
As shown in \Cref{fig:combined_tmp_metrics}(a), \TIE substantially improves all temporal constraint metrics. \Cref{fig:t2v_demo} provides a representative qualitative example: \TIE aligns event onset and multi-event responses more faithfully than prompt concatenation alone. This confirms that \TIE does not merely improve event recognition, but grounds events to the correct temporal support.

\begin{figure*}[htb]
    \centering
    \includegraphics[width=1.0\linewidth]{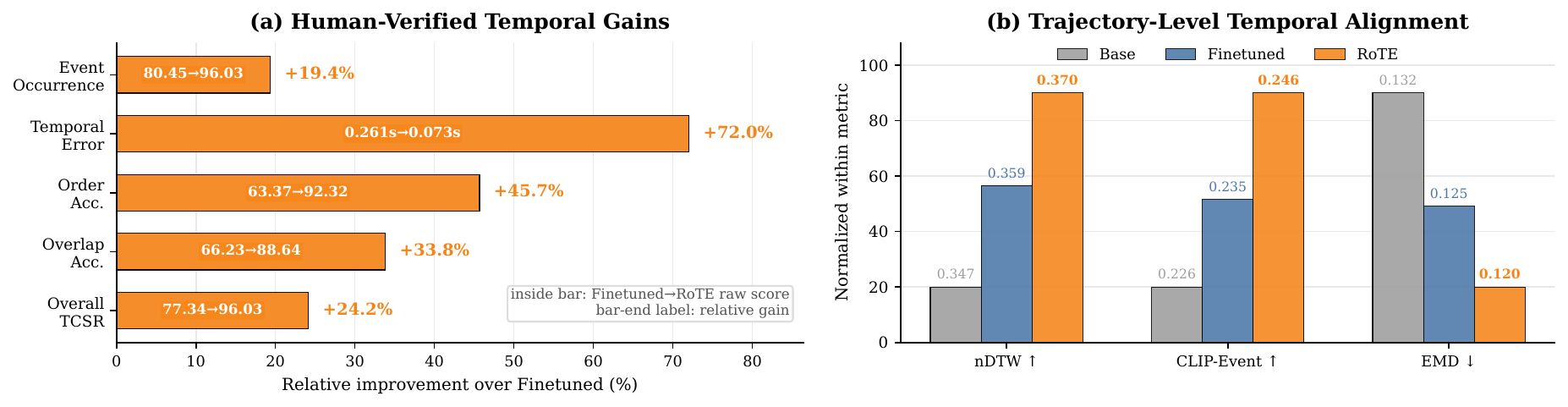}\vspace{-0.1cm}
    \captionsetup{justification=justified, singlelinecheck=true}
    \caption{
\textbf{Temporal alignment improvements of \TIE.}
(a) Human-verified metrics show that RoTE consistently improves over finetuning across event occurrence, temporal error, ordering accuracy, overlap accuracy, and overall TCSR. Each bar shows the relative improvement over the finetuned baseline, while the in-bar label reports the raw score change from finetuning to \TIE.
(b) Trajectory-level metrics further confirm the improvement: \TIE achieves higher nDTW and CLIP-Event scores and lower EMD, indicating better temporal alignment at the trajectory level.
}
\label{fig:combined_tmp_metrics}
\end{figure*}

\begin{figure*}[htb]
    \centering
    \hspace{30pt}
    \includegraphics[width=1.0\linewidth]{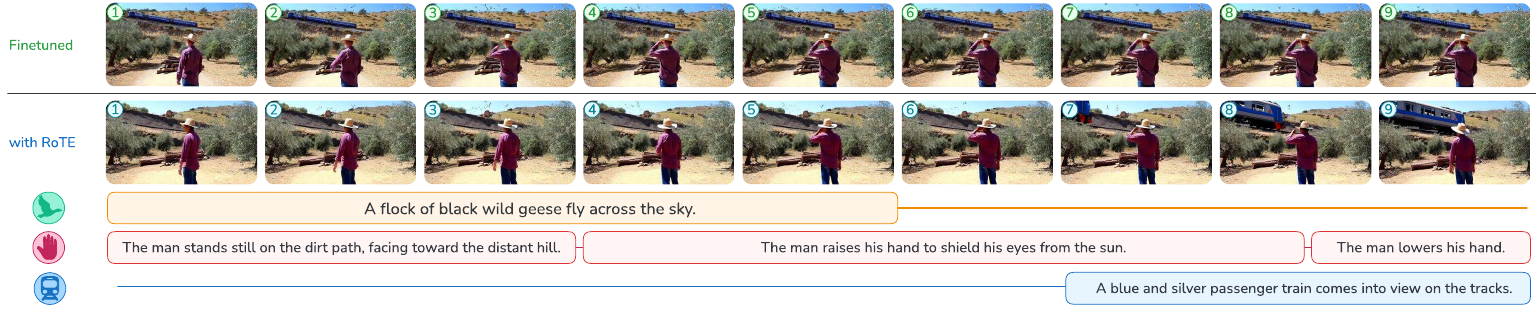}\vspace{-0.7em}
    \captionsetup{justification=justified, singlelinecheck=true}
    \caption{Text-to-video generation example. The \textbf{Finetuned} baseline misses some requested events (e.g., ``the man puts his hand off'') and misaligns event onset (e.g., the train remains static in the frame instead of entering at 7.0s), whereas the \TIE-enhanced variant responds at the correct time.}
    \label{fig:t2v_demo}
    \vspace{-0.2cm}
\end{figure*}
\begin{table}[htb]
\centering
\small
\setlength{\tabcolsep}{6pt}
\begin{tabular}{c|ccccc|cc|c}
\hline
Method & Human &	Animal 	&Object 	& Retrieval & 	Creative 	& Easy & 	Hard &	Average\\ \hline
Base (10s) & 27.6\% &	27.2\% 	&20.3\% 	& 43.1\% &	13.9\% &	34.5\% &	5.9\% &	23.9\%\\
Finetuned (10s)  &	39.0\% &	41.0\% &	35.7\% 	&53.1\%& 	25.3\% &	55.5\% &	21.5\% &	38.8\%\\
{\lmrope} (10s) & 58.5\% &	52.3\% &	48.4\% &	53.8\% &	42.7\% &	64.4\% &	38.8\% &	52.3\%\\ 
Base (5s) & 41.1\% &	41.7\% &	35.9\% &	51.2\% &	28.0\% &	58.6\% &	20.2\% &	39.7\%\\
TS-Attn\cite{zhang2026tsattntemporalwiseseparableattention} (5s) & 42.2\% & 45.4\% & 36.6\% & 52.8\% & 33.0\% & 56.9\% &25.5\% & 41.8\%\\ 
Finetuned (5s)& 41.9\% &	39.8\% &	37.1\%& 	50.8\% 	&26.7\% &	55.8\% &	23.0\% &	39.3\%\\
{\lmrope} (5s) & 54.7\% &	52.8\% &	42.4\% 	& 53.1\% &	36.3\% &	61.2\% &	36.9\% &	49.4\%\\ \hline
\end{tabular}
\caption{StoryEval evaluation results under 10s and 5s generation settings. 
We include both settings because our main setup generates 10s videos (161 frames, 16 fps), whereas Wan2.2-5B primarily supports 5s generation (121 frames, 24 fps). 
All methods are evaluated at 720P. 
TS-Attn is reproduced at 720P instead of its original 480P setting because Wan2.2-5B does not support 480P synthesis.}\vspace{-0.3cm}
\label{tab:expv3_storyeval}
\end{table}\vspace{-0.2cm}


\paragraph{Trajectory-Level Temporal Alignment.}

Human-verified TCSR directly evaluates event-level constraint satisfaction. We complement it with trajectory-level temporal alignment metrics with nDTW \cite{sakoe1978dynamic}, EMD and CLIP-Event.
Given a generated video and a reference event timeline, nDTW computes the distance of the optimal temporal warping path between the generated sequence and the target sequence, EMD calculates the Wasserstein distance between generated videos and ground truths on the temporal axis, while CLIP-Event calculates the per-event frame-level video CLIP Score. Higher nDTW and CLIP-Event indicate higher temporal alignment. Lower EMD represents closer results in temporal event distribution.
As shown in Figure \ref{fig:combined_tmp_metrics} (b), \TIE consistently improves nDTW, EMD and CLIP-Event. Combined with the TCSR results above, this shows that \TIE improves temporal controllability from both event-level and trajectory-level perspectives.

\paragraph{Comparison on Baselines' Native Regime.}
\omnievent~experiments above target the concurrent regime where MinT and TS-Attn are structurally inapplicable. We therefore evaluate \TIE on StoryEval~\cite{wang2024storyeval} and StoryBench~\cite{bugliarello-etal-2023-storybench}—sequential-event benchmarks built for prior work, where baselines operate in their designed regime.
As shown in Table \ref{tab:expv3_storyeval}, we can see that, under \text{OOD} prompt (the prompt of StoryEval is out-of-distribution from our training data) and \text{OOD} length (we mainly train the model at 10s video clips) condition, Finetuned fails to improve over baseline, while \TIE successfully models the temporal organization.
Comparison results on StoryBench \cite{bugliarello-etal-2023-storybench} can be found in \Cref{append:moreresults}.
These results further show that \textbf{TIE}'s temporal-modeling advantage holds under broader evaluation protocols. Additional evaluations on StoryBench and prompt-extension settings in \Cref{append:moreresults} show consistent gains under broader story-generation and prompt-conditioning protocols.

\paragraph{Faster Convergence of Temporal Grounding.}
Beyond final accuracy, we evaluate whether \TIE accelerates the learning of temporal grounding. We train Finetuned and \TIE variants under identical data (\textit{RoboticsEvents}), batch size, learning rate, and training schedule, and measure temporal metrics at matched training steps.
As shown in \Cref{fig:convergence}, TIE reaches the same temporal-control level (nDTW, EMD) and visual quality (FVD, CLIP-Video) (it is noted that CLIP-Video is different from CLIP-Event) in substantially fewer steps. This suggests that interval-aware encoding improves the efficiency of learning temporal grounding, rather than merely adding capacity or overfitting to the final evaluation metrics.

\begin{figure*}[htb]
    \centering
    \includegraphics[width=1.0\linewidth]{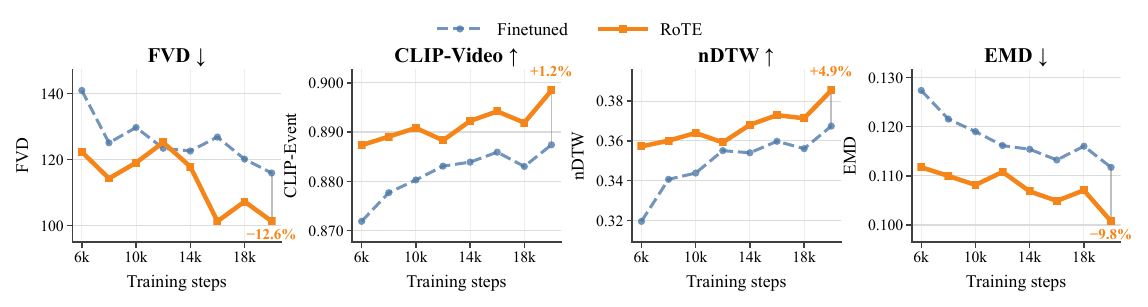}\vspace{-0.4cm}
    \captionsetup{justification=justified, singlelinecheck=true}
    \caption{Performance trends during training with and without \TIE. \TIE leads to faster improvement across both visual and temporal metrics.}
    \label{fig:convergence}
    \vspace{-0.4cm}
\end{figure*}

\begin{table*}[htb]
\centering
\small
\setlength{\tabcolsep}{4pt}
\begin{tabular}{c|cc|c|ccccc}
\hline
Method & FID$\downarrow$ & FVD$\downarrow$ & CLIP-Event $\uparrow$ & VQ$\uparrow$ & TC$\uparrow$ & DD$\uparrow$ & TA$\uparrow$ & FC$\uparrow$\\ \hline
Base & 59.68 & 357.51 & 0.226 & 2.73 & 2.78 & 2.58 & 2.68 & 2.65\\
NoRoPE & 43.74 & 234.40 & 0.235 & 3.03 & 3.01 & 2.97 & 2.86 & 2.94\\
{\dote} & 43.84 & 234.79 & 0.241 & 3.05 & 3.01 & 3.01 & 2.88 & 2.94\\
{\lmrope} & \textbf{42.53} & \textbf{217.29} & \textbf{0.246} & \textbf{3.10} & \textbf{3.05} & \textbf{3.06} & \textbf{2.92} & \textbf{2.99}\\ \hline
$\gamma = 2.0$ & 43.00 & 231.20 & 0.246 & 3.08 & 3.04 & 3.04 & 2.90 & 2.98\\
$\gamma = 8.0$ & 43.15 & 249.24 & 0.245 & 3.09 & 3.05 & 3.06 & 2.92 & 2.99\\ \hline
\end{tabular}
\caption{Ablation study on \textit{PexelsEvents}. The comparison between \textbf{NoRoPE}, \dote, and \TIE isolates the effect of interval-aware integration over prompt concatenation and boundary-only encoding.}
\label{tab:expv3_ablation}
\vspace{-0.5em}
\end{table*}



\paragraph{Case Study: Complex Multi-Subject Interaction.}

We evaluate \TIE on gameplay data from \textit{GameEvents} to test fine-grained action-centric event grounding in a high-dynamic multi-subject domain. Each sample contains structured event intervals for the player and boss. Unlike general web videos, this domain provides highly accurate event boundaries from game instrumentation, allowing us to test whether the model respects precise action timing.
\Cref{fig:eldenring_controlling} shows controlled prompt modifications applied to later event intervals while keeping earlier intervals fixed. The generated videos share the same early trajectory, then diverge according to the modified future event descriptions. In one sequence, the Tarnished rushes forward and attempts an attack but is interrupted by the boss. In another, the boss changes attack type, producing a different visual effect. In a third, the Tarnished rolls backward and avoids the hit. This demonstrates that \TIE supports localized temporal editing: changing a future interval affects the intended future behavior without corrupting earlier events.

\begin{figure*}[htb]
    \centering
    \hspace{30pt}
    \includegraphics[width=1.0\linewidth]{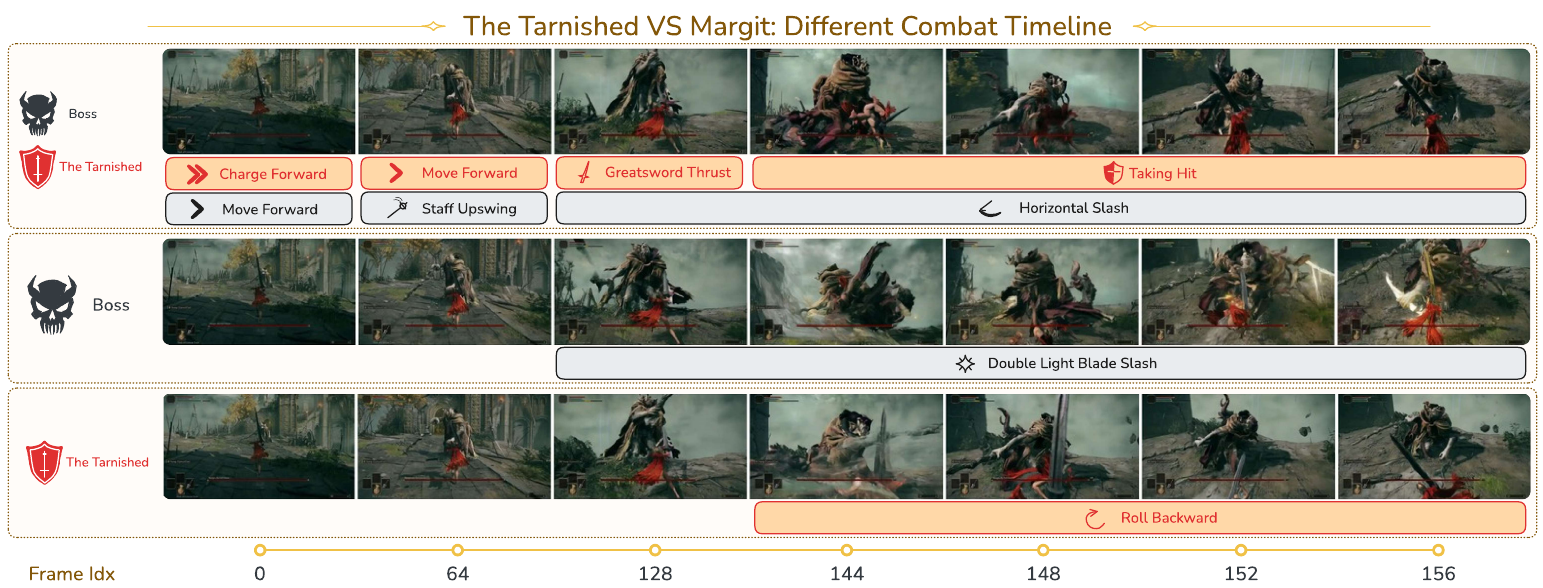}\vspace{-0.1cm}
    \captionsetup{justification=justified, singlelinecheck=true}
    \caption{Generation Results on Game Data: In each row we show the frames and their actions. To illustrate the controllability enabled by \textbf{TIE}, we modify the target events and obtain different endings.}
    \label{fig:eldenring_controlling}
    \vspace{-1.2em}
\end{figure*}




We also performed experiments on \textit{RoboticsEvents}. Visual results can be found in \Cref{sec:robo_case}.
\subsection{Ablation Study}

Finally, we conduct minimal ablations to verify the components implied by the \TIE formulation. Since \TIE is a simple interval-encoding module rather than a multi-component system, we focus on the variants that directly correspond to the theory. \textbf{NoRoPE} removes temporal encoding and relies on prompt concatenation. \textbf{DoTE} encodes only interval boundaries using a Dirac-kernel variant, preserving start/end timestamps while discarding interior evidence. \lmrope uses uniform-kernel integration over the full interval with normalization. We also test sensitivity to the scaling factor $\gamma$.
The comparison between \textbf{NoRoPE} and \dote shows that explicitly encoding temporal boundaries already improves temporal alignment. However, \dote remains weaker than \TIE because it ignores the interval interior. The full \lmrope version achieves the best FVD and CLIP-Event, supporting the theoretical claim that interval integration is more appropriate than boundary-only encoding. The $\gamma$ variants remain close to the full model, suggesting that the method is mostly hyperparameter-agnostic.

\subsection{Robustness to noisy temporal annotations.}
\begin{figure*}[htb]
    \centering
    \includegraphics[width=1.0\linewidth]{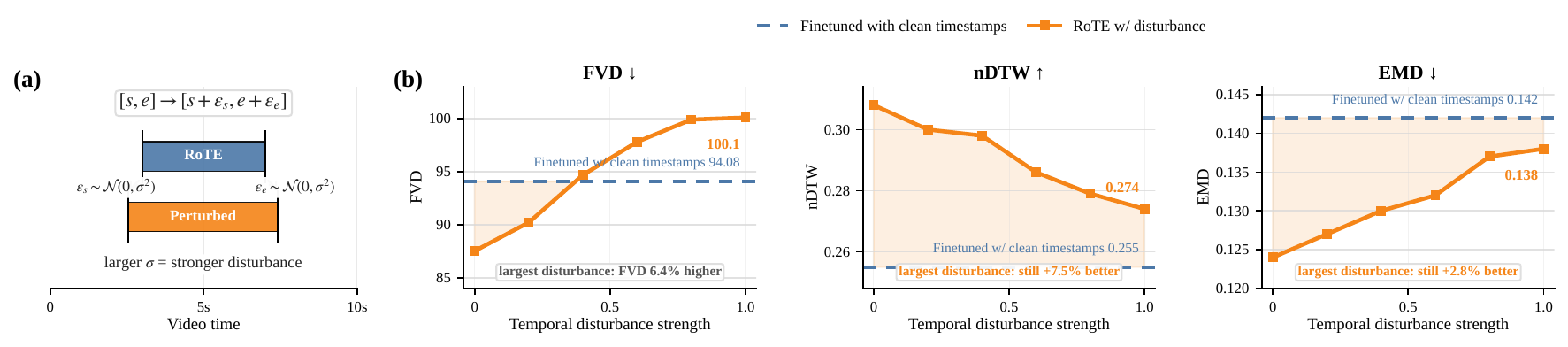}\vspace{-0.15cm}
    \captionsetup{justification=justified, singlelinecheck=true}
    \caption{\textbf{Robustness to temporal annotation disturbance.}
(a) We perturb the start and end timestamps of each \lmrope interval with Gaussian noise:
$[s,e]\rightarrow[s+\epsilon_s,e+\epsilon_e]$, where
$\epsilon_s,\epsilon_e\sim\mathcal{N}(0,\sigma^2)$.
(b) \lmrope remains robust as the disturbance strength increases: even under the largest disturbance, it still achieves better temporal alignment than finetuning with clean timestamps, while only moderately increasing FVD.}
    \label{fig:temp_dist}
\end{figure*}
In real-world data collection, perfectly precise temporal boundaries are rarely available at scale. 
Unlike controlled game traces or carefully curated robotic demonstrations, large-scale event annotations for open-domain videos will often need to be obtained from VLMs, action detectors, captioning models, or other automatic perception systems. 
Such annotations inevitably contain boundary noise: the event may be correctly identified, but its start and end timestamps may be shifted by several frames or even longer. 
Therefore, robustness to imperfect interval endpoints is not merely a desirable property, but a prerequisite for using interval-aware temporal encoding in large-scale video pretraining.

To test this property, we perturb the start and end timestamps of each RoTE interval with Gaussian noise and evaluate how temporal alignment changes as the noise strength increases. 
The results are shown in Figure~\ref{fig:temp_dist}. 
\lmrope remains stable under increasingly noisy boundaries: even at the largest perturbation level, it still outperforms the finetuning baseline trained with clean timestamps on nDTW and EMD. 
This indicates that the benefit of \lmrope does not rely on exact boundary supervision. 
Instead, \lmrope uses the interval as a coarse temporal support for event grounding, and its duration-normalized interval aggregation provides a robust inductive bias even when the annotated endpoints are imprecise.

This robustness is important for scalability. 
A temporal representation that requires frame-perfect event boundaries would be difficult to deploy beyond small curated datasets, because large-scale real-world event annotations are necessarily approximate. 
By remaining effective under noisy start and end times, \lmrope makes interval-conditioned training compatible with realistic annotation pipelines, where event intervals may come from automatic VLM-based labeling or other neural detectors rather than manual frame-level supervision.



\subsection{Summary of Experimental Findings}
The experiments provide consistent evidence for the effectiveness of \TIE.
First, \TIE is compatible with strong pretrained DiT models: it preserves visual quality and semantic alignment while improving temporal grounding. 
Second, \TIE improves precise interval controllability, with gains verified by human metrics, trajectory-level metrics, and out-of-distribution story-generation benchmarks.
Third, \TIE remains robust when dealing with noisy timestamp, indicating robustness to imperfect interval boundaries.
Fourth, qualitative results on game and robotics scenarios show that \TIE supports multi-subject interactions and future-event editing. 
Finally, ablations confirm that the gains come from interval-aware integration and duration-normalized encoding, rather than generic timestamp conditioning or boundary-only heuristics.
Together, these results show that \TIE is a principled, practical, and robust interval-aware temporal encoding method for event-centric video generation. 
    \section{Conclusion}
This paper proposes Time Interval Encoding (\TIE), an interval-aware formulation for the concurrent-event regime that single-active-prompt methods cannot represent. By modeling events as time intervals rather than as the point-wise, disjoint slots assumed by prior multi-event generators, \TIE enables precise control over event duration, ordering, and simultaneous overlap.
To demonstrate its practical effectiveness, we evaluate \TIE within a DiT-based video generation setting.
Experiments on both general-domain and domain-specific video data show that \TIE improves temporal alignment and the modeling of highly dynamic motions, particularly in complex multi-event and multi-subject scenarios.
Overall, \TIE provides a principled temporal foundation for controllable video generation and suggests a promising direction for long-horizon generation and interactive world simulation.
We believe \TIE may serve as a useful building block for future generative models that require structured reasoning over time.

\section*{Acknowledgements}

We thank Pexels for providing a large collection of publicly accessible videos that supports the construction of PexelsEvents. We also acknowledge \textit{Elden Ring} and the AgiBot dataset, which provide valuable sources for studying high-dynamic gameplay interactions and robotics event structures. We thank the Cheat Engine community and tooling ecosystem for enabling game-state inspection during data collection. We are grateful to the external annotation team for their careful human verification efforts, and to Xiangrui Ke for helpful suggestions and discussions on data annotation.
\section*{Impact Statement}

This paper presents work whose goal is to advance the field of machine learning, specifically focusing on \textbf{spatiotemporal modeling in video generation}. 
The proposed method contributes to the broader scientific community by \textbf{improving multi-subject control and temporal consistency in multi-event video generation}, with potential applications in \textbf{controllable video creation and simulation}. 




    \bibliographystyle{plain}
    \bibliography{defines/tie}
}{}
\iftoggle{isAppendix}{
\appendix
\renewcommand{\theequation}{A\arabic{equation}}
\renewcommand{\thefigure}{A\arabic{figure}}
\renewcommand{\thetable}{A\arabic{table}}
\renewcommand{\theHequation}{A\arabic{equation}}
\renewcommand{\theHfigure}{A\arabic{figure}}
\renewcommand{\theHtable}{A\arabic{table}}
\setcounter{equation}{0}
\setcounter{figure}{0}
\setcounter{table}{0}
\section{Details and Proofs for \TIE}
\label{append:tie}
\subsection{Proofs for Section~3}
\label{append:proofs_tie}
\label{append:derivation}
This subsection provides the full formalization and proofs for the principle-based characterization stated in \cref{sec: RoTE}, and proves \cref{cor:rote_uniform,cor:dote_dirac}.

\begin{definition}[Formal principles for interval-aware scores]\label{def:tie_principles}
Let $s_{\mathrm{rope}}(\tokenq,\tokenk;t_q,\tau)$ denote the standard point-to-point \rope score in \cref{eqn:rope_relative}. We say that an interval-aware score $s(\tokenq,\tokenk;t_q,I)$ over an interval $I=[t^s,t^e]$ satisfies the following two principles:
\begin{description}
    \item[Temporal Integrability, i.e.,] There exists a probability $\mu_I$ supported on $I$ with $\mu_I(I)>0$ such that the \emph{raw} interval score is obtained by integrating point-wise \rope scores over the temporal support of the interval,
    \begin{equation}
        \bar{s}(\tokenq,\tokenk;t_q,I)
        :=
        \mathbb E_{\tau\sim\mu_I} [s_{\mathrm{rope}}(\tokenq,\tokenk;t_q,\tau)].
        \label{eqn:tie_temporal_integration}
    \end{equation}
    \item[Duration Invariance, i.e.,] The final interval-aware score is obtained from the raw integral by a positive scalar normalization,
    \begin{equation}
        s(\tokenq,\tokenk;t_q,I)
        :=
        \frac{1}{C(\mu_I)}\bar{s}(\tokenq,\tokenk;t_q,I),
        \qquad C(\mu_I)>0,
        \label{eqn:tie_duration_normalization}
    \end{equation}
    where $C(\mu_I)$ depends only on the interval measure $\mu_I$ and is independent of $\tokenq$, $\tokenk$, and $t_q$.
\end{description}
When both conditions hold, we call $s$ a \emph{principle-consistent} point-to-interval score.
\end{definition}

\begin{theorem}[Uniform-kernel \TIE yields $\lmrope$]\label{cor:rote_uniform}
For the uniform kernel $\mu_I=\Uniform([t^s,t^e])$ with interval center $c=(t^s+t^e)/2$ and radius $r=(t^e-t^s)/2$, the \TIE encoder takes the closed form
\begin{equation}
    {\mathbf{R}}_{c,r} = \operatorname{diag}({\mathbf{A}}_{1,c,r}, {\mathbf{A}}_{2,c,r}, \dots, {\mathbf{A}}_{{d/2},c,r}),
\end{equation}
and therefore reduces to the normalized interval encoder
\begin{equation}
    \mathrm{RoTE}(\tokenk, c, r) = \frac{1}{{C}_{r}} {\mathbf{R}}_{c,r} \tokenk.
\end{equation}
Consequently, for a query token $\tokenq_i$ at time $m_i$ and a key token $\tokenk_j$ associated with interval $I_j=[t_j^s,t_j^e]$, the \TIE attention score becomes
\begin{equation}
    \attnscore_{i,j}
    \propto
    \mathrm{RoPE}(\tokenq_i,m_i)^\top\mathrm{RoTE}(\tokenk_j,c_j,r_j),
    \qquad
    c_j=\frac{t_j^s+t_j^e}{2},\;
    r_j=\frac{t_j^e-t_j^s}{2}.
\end{equation}
\end{theorem}

\begin{proof}[Proof of \cref{cor:rote_uniform}]
Set $\mu_I=\Uniform([t^s,t^e])$ and write the interval as $I=[c-r,c+r]$. By \cref{eqn:tie_single_kernel},
\begin{equation}
    \mathrm{TIE}(\tokenk,I)
    =
    \frac{1}{ C_{r}}
    \frac{1}{2r}
    \int_{c-r}^{c+r}\rope(\tokenk,\tau)\diff\tau.
    \label{eqn:appendix_uniform_tie}
\end{equation}
For each frequency $\theta$,
\begin{equation}
    \frac{1}{2r}\int_{c-r}^{c+r}e^{\mathrm{i}\theta\tau}\diff\tau
    =
    e^{\mathrm{i}\theta c}\sinc(\theta r),
\end{equation}
so every $2\times 2$ \rope block averages to ${\mathbf{A}}_{i,c,r}=\sinc(\theta_i r)\mathbf{A}_{i,c}$. Stacking the blocks yields the final results.

To enforce Duration Invariance, we normalize by requiring unit expected self-overlap at the interval center:
\begin{equation}
    1
    =
    \E_{\tokenq\sim\Uniform(S^{d-1})}
    \left[
        \mathrm{RoPE}(\tokenq,c)^\top\mathrm{RoTE}(\tokenq,c,r)
    \right].
\end{equation}
Using the isotropy identity $\E_{\tokenq\sim\Uniform(S^{d-1})}[\tokenq^\top M\tokenq]=\frac{1}{d}\operatorname{tr}(M)$, we obtain
\begin{equation}
    1
    =
    \frac{1}{C_{r}}
    \frac{2}{d}
    \sum_{\ell=1}^{d/2}\sinc(\theta_\ell r),
\end{equation}
Substituting the resulting encoder into the $\TIE$ attention score yields the attention score \cref{eqn:tie_score} in \cref{eqn:rote_score}
\end{proof}

\begin{theorem}[Dirac-kernel \TIE yields \dote]\label{cor:dote_dirac}
If \TIE is instantiated with Dirac kernels concentrated at the interval boundaries, i.e., $\mu_{I,s}=\delta_{t^s}$ and $\mu_{I,e}=\delta_{t^e}$ on the two channel groups of $\tokenk=\tokenk^{(s)}\oplus\tokenk^{(e)}$, then the resulting interval encoder is exactly the boundary-only form \dote:
\begin{equation}
    \mathrm{DoTE}(\tokenk, I) = \mathbf{Concat}(\mathrm{RoPE}(\tokenk^{(s)},t^s),\mathrm{RoPE}(\tokenk^{(e)},t^e)).
\end{equation}
\end{theorem}

\begin{proof}[Proof of \cref{cor:dote_dirac}]
Let $\tokenk=\tokenk^{(s)}\oplus\tokenk^{(e)}$ and instantiate \TIE on the two channel groups with the Dirac kernels $\mu_{I,s}=\delta_{t^s}$ and $\mu_{I,e}=\delta_{t^e}$. Because each Dirac kernel has unit mass, the corresponding normalization is $C(\mu_{I,s})=C(\mu_{I,e})=1$. Applying \TIE to the two subspaces gives
\begin{equation}
    \mathrm{TIE}_{s}(\tokenk^{(s)},I)=\mathrm{RoPE}(\tokenk^{(s)},t^s),
    \qquad
    \mathrm{TIE}_{e}(\tokenk^{(e)},I)=\mathrm{RoPE}(\tokenk^{(e)},t^e).
\end{equation}
Concatenating the two channel groups yields
\begin{equation}
    \mathrm{TIE}(\tokenk,I)
    =
    \mathbf{Concat}(\mathrm{RoPE}(\tokenk^{(s)},t^s),\mathrm{RoPE}(\tokenk^{(e)},t^e))
    =
    \mathrm{DoTE}(\tokenk,I),
\end{equation}
which is exactly the $\mathrm{DoTE}$. Since the construction keeps only two boundary atoms, it is a boundary-only specialization of \TIE and does not average over the interval interior.
\end{proof}

\begin{theorem}[Expected robustness of \lmrope to timestamp noise]
Define the unscaled interval-attention score
\begin{equation}
    \attnscore_{i,j}(c,r)
    :=
    \mathrm{RoPE}(\tokenq_i,m_i)^\top \mathrm{RoTE}(\tokenk_j,c,r),
    \qquad
    \mathrm{RoTE}(\tokenk,c,r)=C_r^{-1}\mathbf{R}_{c,r}\tokenk .
\end{equation}
For event $I_j=[t_j^s,t_j^e]$, let $c_j=(t_j^s+t_j^e)/2$ and
$r_j=(t_j^e-t_j^s)/2>0$. Suppose its observed endpoints are
\begin{equation}
    \widetilde t_j^s=t_j^s+\epsilon_j^s,
    \qquad
    \widetilde t_j^e=t_j^e+\epsilon_j^e,
    \qquad
    |\epsilon_j^s|,|\epsilon_j^e|\le \delta<r_j
    \quad\text{a.s.}
\end{equation}
Write
\begin{equation}
    \Delta c_j:=\frac{\epsilon_j^s+\epsilon_j^e}{2},
    \qquad
    \Delta r_j:=\frac{\epsilon_j^e-\epsilon_j^s}{2},
    \qquad
    \widetilde c_j=c_j+\Delta c_j,
    \qquad
    \widetilde r_j=r_j+\Delta r_j .
\end{equation}
Assume the normalization does not degenerate on the perturbed radius range,
\begin{equation}
    C_{\min,j}:=
    \inf_{\rho\in[r_j-\delta,r_j+\delta]} |C_\rho|>0 .
\end{equation}
Then
\begin{equation}
\left|
\attnscore_{i,j}(\widetilde c_j,\widetilde r_j)
-
\attnscore_{i,j}(c_j,r_j)
\right|
\le
\frac{\left\lVert\tokenq_i\right\rVert\left\lVert\tokenk_j\right\rVert\delta}{r_j-\delta}
\left(
\frac{3}{C_{\min,j}}+
\frac{2}{C_{\min,j}^2}
\right).
\label{eqn:rote_expected_noise_bound_simple}
\end{equation}
Therefore, if $\delta\le r_j/2$ and $C_{\min,j}$ is bounded away from zero, \lmrope has expected
noise sensitivity $O(\left\lVert\tokenq_i\right\rVert\left\lVert\tokenk_j\right\rVert\delta/r_j)$. By contrast,
point-wise \rope and boundary-only \dote have local worst-case timestamp sensitivity
$O(\left\lVert\tokenq_i\right\rVert\left\lVert\tokenk_j\right\rVert\theta_{\max}\delta)$, where
$\theta_{\max}:=\max_\ell\theta_\ell$, with no decay in the interval radius $r_j$.
\end{theorem}

\begin{proof}[Proof of \cref{thm:rote_noise_robustness}]
Let
\begin{equation}
    \Phi(c,r):=C_r^{-1}\mathbf{R}_{c,r},
    \qquad
    \attnscore_{i,j}(c,r)
    =
    \mathrm{RoPE}(\tokenq_i,m_i)^\top \Phi(c,r)\tokenk_j .
\end{equation}
Since \rope is an orthogonal block rotation,
$\left\lVert\mathrm{RoPE}(\tokenq_i,m_i)\right\rVert=\left\lVert\tokenq_i\right\rVert$.

We first bound the derivatives of $\Phi$. For each frequency block,
\begin{equation}
    \mathbf{R}_{c,r}^{(\ell)}=
    \sinc(\theta_\ell r)
    \begin{bmatrix}
        \cos(\theta_\ell c) & -\sin(\theta_\ell c)\\
        \sin(\theta_\ell c) & \cos(\theta_\ell c)
    \end{bmatrix}.
\end{equation}
The rotation matrix has operator norm one and $|\sinc(x)|\le 1$, so
\begin{equation}
    \left\lVert\mathbf{R}_{c,r}\right\rVert_{\mathrm{op}}\le 1 .
\end{equation}
For $r>0$, blockwise differentiation gives
\begin{equation}
    \left\lVert\partial_c \mathbf{R}_{c,r}\right\rVert_{\mathrm{op}}
    =
    \max_\ell \theta_\ell |\sinc(\theta_\ell r)|
    =
    \max_\ell \frac{|\sin(\theta_\ell r)|}{r}
    \le \frac{1}{r} .
\label{eqn:appendix_dc_R_bound}
\end{equation}
Similarly, since
\begin{equation}
    \theta\,\sinc'(\theta r)
    =
    \frac{\cos(\theta r)-\sinc(\theta r)}{r},
\end{equation}
we have
\begin{equation}
    \left\lVert\partial_r \mathbf{R}_{c,r}\right\rVert_{\mathrm{op}}
    =
    \max_\ell |\theta_\ell\sinc'(\theta_\ell r)|
    \le
    \frac{2}{r} .
\label{eqn:appendix_dr_R_bound}
\end{equation}
Moreover,
\begin{equation}
    C_r=\frac{2}{d}\sum_{\ell=1}^{d/2}\sinc(\theta_\ell r)
    \quad\Longrightarrow\quad
    |C'_r|
    \le
    \frac{2}{d}\sum_{\ell=1}^{d/2}|\theta_\ell\sinc'(\theta_\ell r)|
    \le
    \frac{2}{r} .
\label{eqn:appendix_C_derivative_bound}
\end{equation}
On $\rho\in[r_j-\delta,r_j+\delta]$, we have $|C_\rho|\ge C_{\min,j}$ and
$\rho\ge r_j-\delta$. Hence
\begin{equation}
    \left\lVert\partial_c\Phi(c,\rho)\right\rVert_{\mathrm{op}}
    \le
    \frac{1}{C_{\min,j}(r_j-\delta)},
\label{eqn:appendix_dc_Phi_bound}
\end{equation}
and, using
$\partial_r\Phi=C_r^{-1}\partial_r\mathbf{R}_{c,r}-C'_r C_r^{-2}\mathbf{R}_{c,r}$,
\begin{equation}
    \left\lVert\partial_r\Phi(c,\rho)\right\rVert_{\mathrm{op}}
    \le
    \frac{1}{r_j-\delta}
    \left(
        \frac{2}{C_{\min,j}}+
        \frac{2}{C_{\min,j}^2}
    \right).
\label{eqn:appendix_dr_Phi_bound}
\end{equation}

Now fix one realization of the endpoint perturbation. Write
$x=(c_j,r_j)$ and $\Delta=(\Delta c_j,\Delta r_j)$. Since $|\Delta r_j|\le\delta$, every point on the
line segment $x+\tau\Delta$, $\tau\in[0,1]$, has radius in $[r_j-\delta,r_j+\delta]$. By the
fundamental theorem of calculus in the finite-dimensional matrix space,
\begin{align}
    \Phi(x+\Delta)-\Phi(x)
    &=
    \int_0^1 D\Phi(x+\tau\Delta)[\Delta] \,d\tau \nonumber\\
    &=
    \int_0^1
    \left(
        \Delta c_j\,\partial_c\Phi(x+\tau\Delta)
        +
        \Delta r_j\,\partial_r\Phi(x+\tau\Delta)
    \right)
    \,d\tau .
\label{eqn:appendix_line_integral}
\end{align}
Taking operator norms and using the triangle inequality gives
\begin{align}
    \left\lVert\Phi(x+\Delta)-\Phi(x)\right\rVert_{\mathrm{op}}
    &\le
    |\Delta c_j|\sup_{\tau\in[0,1]}
    \left\lVert\partial_c\Phi(x+\tau\Delta)\right\rVert_{\mathrm{op}}
    \nonumber\\
    &\quad+
    |\Delta r_j|\sup_{\tau\in[0,1]}
    \left\lVert\partial_r\Phi(x+\tau\Delta)\right\rVert_{\mathrm{op}}
    \nonumber\\
    &\le
    \frac{1}{r_j-\delta}
    \left[
        \frac{|\Delta c_j|}{C_{\min,j}}
        +
        \left(
            \frac{2}{C_{\min,j}}+
            \frac{2}{C_{\min,j}^2}
        \right)|\Delta r_j|
    \right].
\label{eqn:appendix_Phi_perturb_bound}
\end{align}
This is the operator-valued mean-value inequality; it does not require an equality-form mean value
theorem for the operator norm.

Therefore,
\begin{align}
&\left|
\attnscore_{i,j}(\widetilde c_j,\widetilde r_j)
-
\attnscore_{i,j}(c_j,r_j)
\right| \nonumber\\
&\quad=
\left|
\mathrm{RoPE}(\tokenq_i,m_i)^\top
\left(\Phi(\widetilde c_j,\widetilde r_j)-\Phi(c_j,r_j)\right)
\tokenk_j
\right| \nonumber\\
&\quad\le
\left\lVert\tokenq_i\right\rVert
\left\lVert\tokenk_j\right\rVert
\left\lVert\Phi(\widetilde c_j,\widetilde r_j)-\Phi(c_j,r_j)\right\rVert_{\mathrm{op}} .
\end{align}
 Since
$|\epsilon_j^s|,|\epsilon_j^e|\le\delta$ almost surely, we also have
$|\Delta c_j|,|\Delta r_j|\le\delta$, which gives \cref{eqn:rote_expected_noise_bound_simple}.

Finally, for point-wise \rope, let $\mathbf{R}_t$ be the standard RoPE rotation. Blockwise
differentiation gives
\begin{equation}
    \left\lVert\frac{d}{dt}\mathbf{R}_t\right\rVert_{\mathrm{op}}
    =
    \max_\ell \theta_\ell
    =
    \theta_{\max} .
\end{equation}
Hence the same line-integral argument yields
\begin{equation}
    \left\lVert\mathbf{R}_{t+\epsilon}-\mathbf{R}_t\right\rVert_{\mathrm{op}}
    \le
    \theta_{\max}|\epsilon|,
\end{equation}
and this local worst-case Lipschitz scale is independent of any interval radius. The boundary-only
\dote encoder is a concatenation of two point-wise encoders at $t^s$ and $t^e$, so it has the same
order of timestamp sensitivity and no $1/r_j$ decay.
\end{proof}

\subsection{Algorithm}
Here is the demonstrative code for \TIE:
\begin{lstlisting}[style=pythonstyle]
def RoTE(start, end, theta=10000.0, scaling_factor=4.0, alpha=1.0, dim=128):
  # Assuming start, end are torch.Tensor of shape (b, t)
  center = (start + end) * scaling_factor / 2.0
  radius = (end - start) * scaling_factor / 2.0
  freqs = 1.0 / (theta ** (torch.arange(0, dim, 2).double() / dim)).to(start.device)
  theta = torch.einsum("bi,j->bij", center, freqs)
  cos, sin = theta.cos(), theta.sin()
  phi = torch.einsum("bi,j->bij", radius, freqs)
  sincs = torch.sinc(alpha * phi/torch.pi)
  sincs = sincs / torch.mean(sincs, dim=-1, keepdim=True)
  cos, sin = cos * sincs, sin * sincs
  return cos[:, :, None, :], sin[:, :, None, :]
  
def RoTE_apply(x, cos, sin, num_heads):
  x = rearrange(x, "b s (n u h) -> b s n u h", n=num_heads, u=2)
  x_out = torch.cat([x[:, :, :, 0, :] * cos - x[:, :, :, 1, :] * sin, x[:, :, :, 0, :] * sin + x[:, :, :, 1, :] * cos], dim=-1).flatten(2)
  return x_out.to(x.dtype)
\end{lstlisting}
\label{append:algorithm}
\subsection{\lmrope Visualization}
\Cref{fig:rote_visualization} \& \ref{fig:rope_decaying} show a visualization case for \lmrope decaying. Events with longer duration decay slower along the time interval.
\begin{figure*}[ht]
    \centering
    \includegraphics[width=0.9\linewidth]{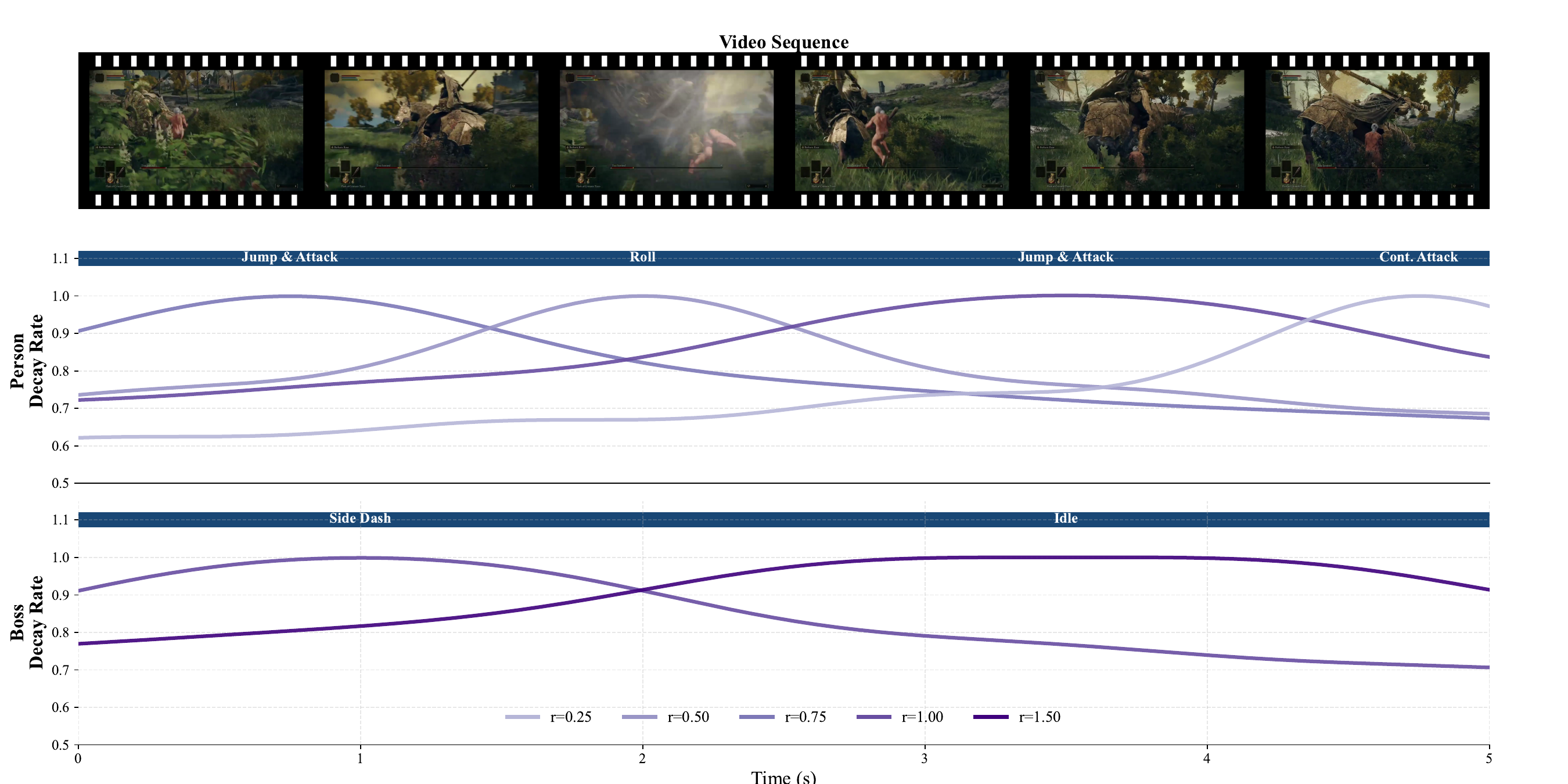}
    \captionsetup{justification=justified, singlelinecheck=true}
    \caption{\lmrope Visualization: The curves show the decaying effect of each event.}
    \label{fig:rote_visualization}
    \vspace{-1.0em}
\end{figure*}
\begin{figure}[htb]
    \centering
    \hspace{30pt}
    \includegraphics[width=0.6\linewidth]{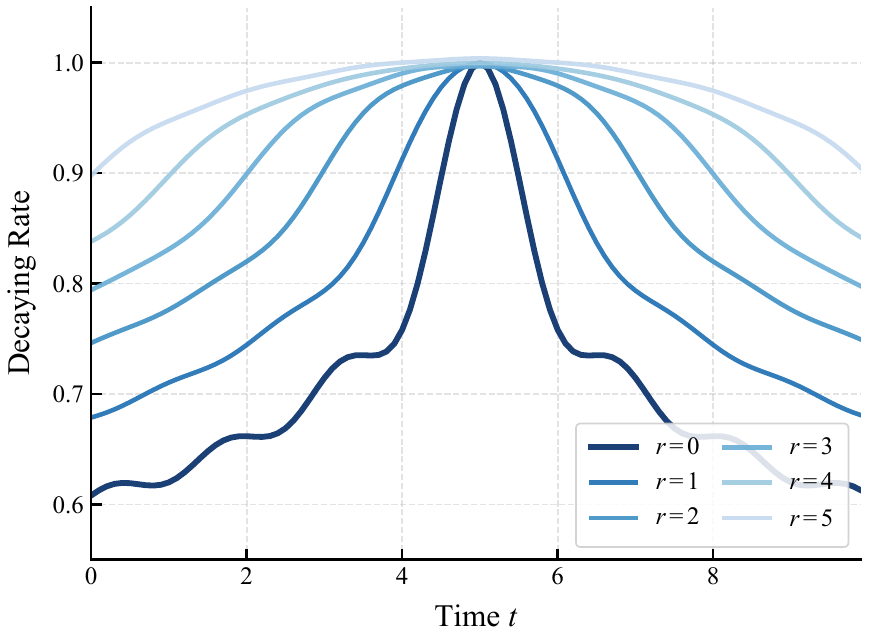}\vspace{-0.8em}
    \captionsetup{justification=justified, singlelinecheck=true}
    \caption{\lmrope Decaying Rate: We visualize the relative decaying rate of our \lmrope when $\gamma=4.0$. For a video clip of 10s long, we fix the middle point of the event at $5.0$s and vary the radius $r$ of event. When $r=0$, \lmrope generally reduces to normal \rope. While a larger $r$ dampens the decaying rate. For a long event that sustains for 10 seconds ($r=5$), the decaying rate remains above 0.9 during the whole event duration. }
    \label{fig:rope_decaying}
\vspace{-1.5em}
\end{figure}
\section{Dataset}
\begin{figure*}[htb]
    \centering
    \includegraphics[width=0.9\linewidth]{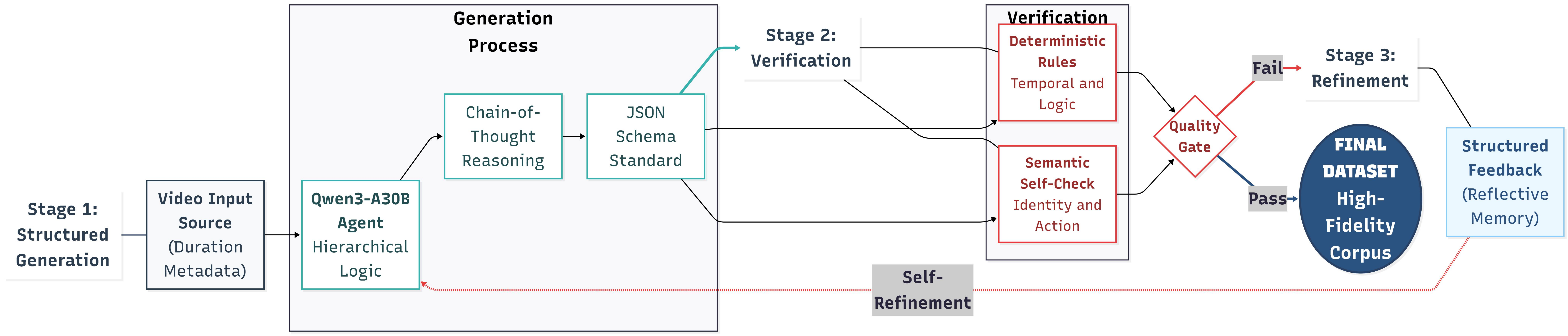}
    \captionsetup{justification=justified, singlelinecheck=true}
    \caption{\textbf{Overview of the Self-Reflective Video Annotation Pipeline.} The system constructs high-fidelity datasets through three iterative stages: (1) \textbf{Structured Generation}, employing Qwen3-A30B with \textbf{CoT reasoning} and hierarchical decomposition under a strict \textbf{JSON schema}; (2) \textbf{Dual-Track Verification}, which enforces \textbf{deterministic constraints} (temporal/logical filters) and \textbf{semantic self-checks} to ensure identity and action fidelity; and (3) a \textbf{Refinement Loop}, where detected failures trigger \textbf{agentic self-correction} via structured error feedback. This workflow effectively eliminates temporal hallucinations and ensures the plausibility of the final corpus.}
    \label{fig:dataset_workflow}
\end{figure*}
\begin{figure*}[http]
    \centering
    \hspace{30pt}
    \includegraphics[width=1.0 \linewidth]{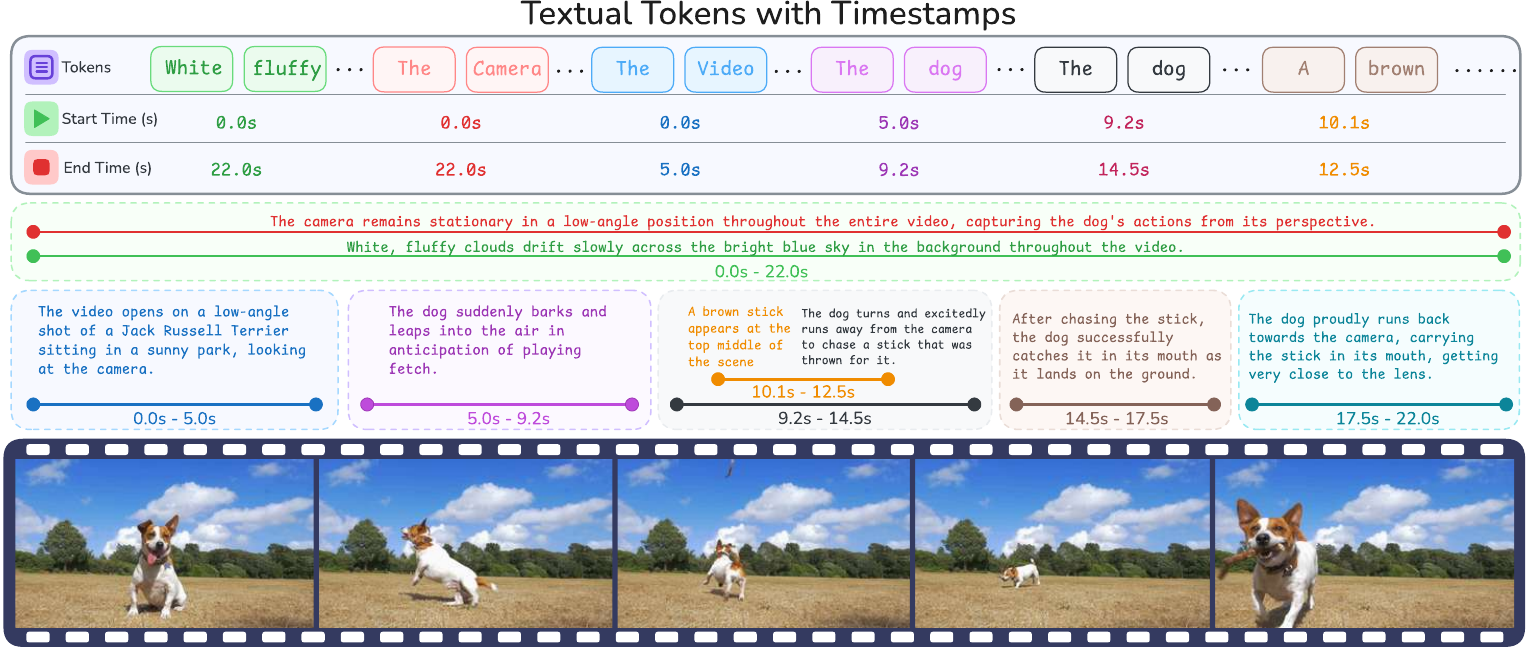}
    \captionsetup{justification=justified, singlelinecheck=true}
    \caption{Textual Tokens with Timestamps: The captions for videos are structured and timestamped, which could be utilized to annotate every token with an accurate start time and end time. In this case, the caption describes the scene, the characters and objects, the background, behaviors and camera motions in detail. Every event is annotated with timestamps of start and end at the accuracy of \(0.25\text{s}\).}
    \label{fig:textual_timestamps}
    \vspace{-0.5em}
\end{figure*}
\subsection{Data Construction via Self-Reflective Agent}
Since interval-based modeling requires temporally precise and structured event supervision, we construct a large-scale dataset using a self-reflective annotation agent.
\subsubsection{The Structured Generation}
We employ a large-scale vision-language model (Qwen3-A30B \cite{qwen3technicalreport}) to generate dense video annotations. To ensure data parseability and consistency, we enforce a strict JSON schema rather than allowing free-form textual output. The prompt is designed to guide the model through a hierarchical decomposition of the video content: it must first explicitly define participant identities (e.g., tracking visual attributes like clothing) and global scene context before generating detailed event timelines. Additionally, we leverage Chain-of-Thought (CoT) prompting \cite{Wei2022ChainOT}, requiring the model to explicitly reason about visual anchors and temporal continuity before finalizing the event logs. Details about the prompt can be found in \cref{append:data}.

\subsubsection{The Verification Criteria}
All verification criteria are designed to enforce consistency with the interval-based event formulation introduced in~\cref{sec: RoTE}.
To mitigate the common issue of temporal \textbf{hallucination} in VLMs, we implement a verification mechanism anchored by ground-truth metadata. Crucially, we explicitly inject the precise video duration (extracted via FFmpeg) into the system prompt. This serves as a hard boundary, preventing the model from generating events that exceed the physical limits of the video.

Building on this ground truth, we enforce a set of deterministic constraints:
\paragraph{Temporal Alignment} All timestamps are strictly quantized to a 0.25s grid. This discretization is specifically designed to align with the temporal granularity of our downstream video generation backbone (operating at 16 fps), ensuring that text-conditioned events map precisely to video frame blocks. \Cref{fig:textual_timestamps} gives an example of the timestamps.
\paragraph{Validity Checks} We apply a series of logical filters to ensure the physical and semantic plausibility of the annotations. First, we enforce that every event's end time must strictly follow its start time, eliminating invalid negative durations. Second, to ensure clear action boundaries, we prohibit overlapping events within the same action track to ensure atomic action definition. This guarantees that distinct actions do not conflict in the timeline. Third, we impose a minimum duration threshold of 1.0s. This filter removes fleeting micro-actions that are too brief for stable video generation, ensuring that every annotated event represents a significant and renderable motion phase.
\paragraph{Semantic Consistency} While rule-based checks ensure structural validity, they cannot detect semantic errors (e.g., hallucinated objects or identity swaps). To address this, we introduce a semantic self-verification step. In this phase, we feed the generated timeline back into the VLM alongside the original video, instructing the model to act as an independent critic. The VLM is tasked with verifying two critical aspects:
\begin{itemize}
    \item \textbf{Identity Consistency} It checks whether the visual attributes described in the text (e.g., "person in the red shirt") remain consistent with the pixel-level visual information throughout the video, flagging any identity switches. 
    \item \textbf{Action Fidelity} It assesses whether the captioned actions (e.g., "picking up a cup") accurately reflect the actual motion dynamics observed in the video frames. If the VLM detects a mismatch, it outputs a specific error description, which triggers the refinement loop.
\end{itemize}
This semantic verification step improves annotation reliability but is not required at inference time.

\subsubsection{The Refinement Loop}
This mechanism serves as the core of our agentic workflow, enabling the system to self-correct rather than simply discarding imperfect generations. When an annotation fails either the programmatic validity checks or the semantic self-verification, the specific error message (e.g., "Event 3 duration 0.5s is less than 1.0s threshold" or "Identity mismatch at 4.5s") is formulated as structured feedback. We inject this feedback directly into the prompt for the next inference cycle, explicitly instructing the VLM to reflect on its previous error.

\subsection{Construction of the  PexelsEvents}
\label{append:data}
Here is our annotation structure for PexelsEvents:
\begin{lstlisting}[language=jsonformat]
{"scene_clarity": "Clear|Cluttered|Chaotic",
"global_caption": {
"short_caption": "Concise summary of scene (NO PEOPLE/ACTION), ~20 words.",
"long_caption": "Detailed description of environment and static objects (NO PEOPLE), ~60 words.",
"visible_objects": ["list", "of", "all", "distinct", "objects", "visible"],
"start_time": 0.0,
"end_time": <duration:.2f>},
"participants": {
"participant_id": {
"short_description": "Short description starting with ID.",
"long_description": "Detailed description starting with ID, distinguishing features.",
"start_time": "float (Exact time of first appearance)",
"end_time": "float (Exact time of last disappearance)",
"timeline": [
{
"distinguishing_feature": "Visual check (e.g., 'red shirt').",
"start_time": float,
"end_time": float,
"is_interaction": boolean,
"interaction_with": ["other_id"],
"short_caption": "Verb + object...",
"long_caption": "Detailed phase log starting with ID. Max 50 words. Include start/end positions and specific body part used."
}
]
}
},
"statistics": {
"video_quality_score": 1-5,
"max_events_count": int,
"participants_with_max_events": ["id"],
"is_suitable_for_learning": boolean
}
}
\end{lstlisting}
Figure \ref{fig:dataset_case} shows an example from PexelsEvents.
\begin{figure*}[htb]
    \centering
    \includegraphics[width=0.9\linewidth]{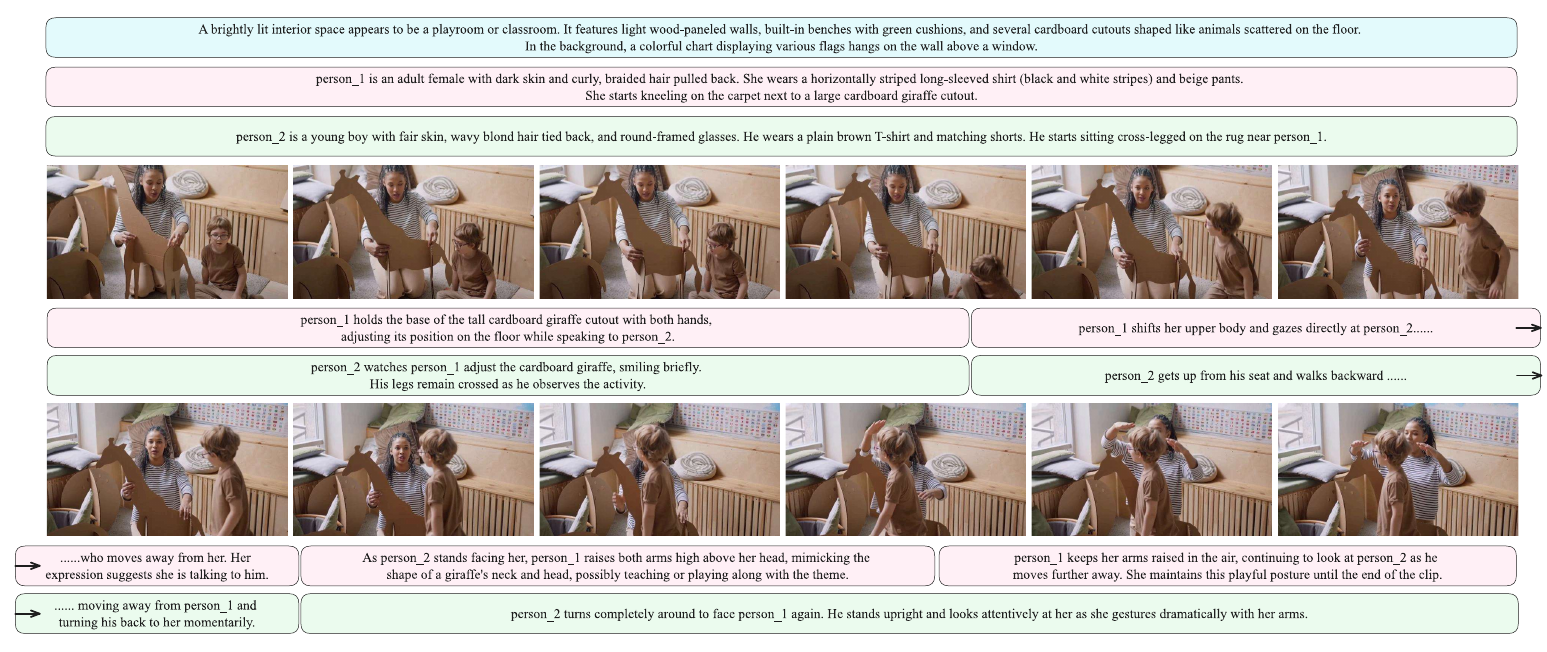}
    \captionsetup{justification=justified, singlelinecheck=true}
    \caption{An example from \textit{PexelsEvents}: global descriptions and event-level descriptions are provided.}
    \label{fig:dataset_case}
    \vspace{-1.5em}
\end{figure*}
\subsection{Construction of the GameEvents}
We collect 128 hours of Elden Ring gameplay videos. Videos are recorded by OBS. All behaviors of player and bosses are recorded by reading the real-time animation ID from the game memory. Then we transform these records into 80k 10-second video clips with structured annotations. \Cref{fig:eldenring_data} shows an example data containing high-dynamic fast-speed interactions of the player and the boss.

\begin{figure*}[htb]
    \centering
    \includegraphics[width=1.0\linewidth]{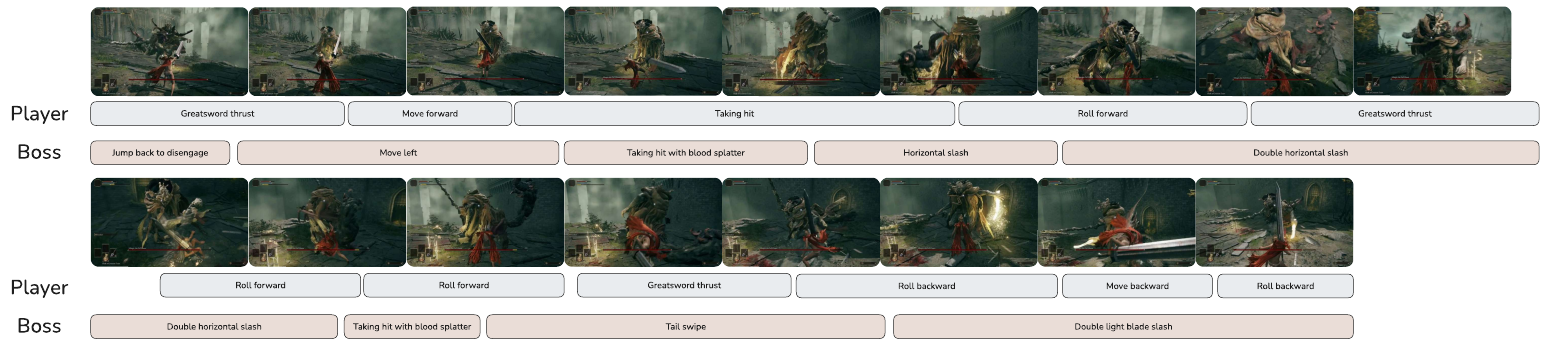}
    \captionsetup{justification=justified, singlelinecheck=true}
    \caption{An example from \textit{GameEvents}: representative animation sequences with detailed event annotations.}
    \label{fig:eldenring_data}
    \vspace{-1.0em}
\end{figure*}
\subsection{Construction of the RoboticsEvents}
We select 86k 10-second video clips from Agibot-World dataset. Which covers several manipulation tasks (including some scenes with human intervention). We filter the recorded trajectories and split them into sub-pieces. We feed VLM with information of these trajectory segments to generate more accurate structured prompts that describe behaviors of each arm and gripper. \Cref{fig:robot_data} shows an example data containing complex interactions of robot and human operator.
\begin{figure*}[htb]
    \centering
    \includegraphics[width=1.0\linewidth]{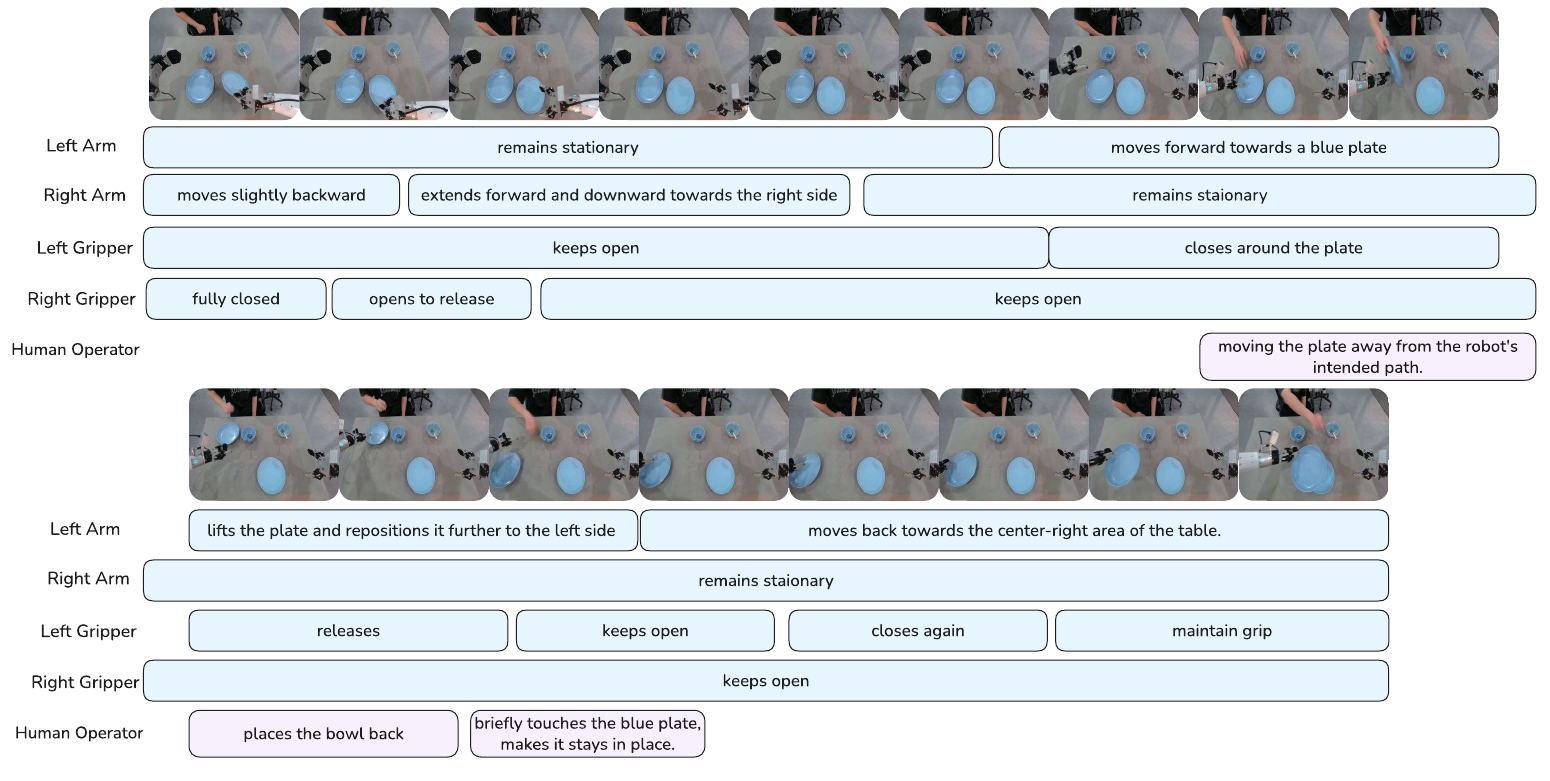}
    \captionsetup{justification=justified, singlelinecheck=true}
    \caption{An example from \textit{RoboticsEvents}: representative manipulation sequences with detailed event annotations.}
    \label{fig:robot_data}
    \vspace{-1.0em}
\end{figure*}
\subsection{Dataset Statistics}
We report several statistics about our \omnievent~dataset in \Cref{tab:data_stat}. The Overlap Prob. indicates the probability that a single event could happen with another one at the same time. These three parts of \omnievent~covers from low-speed (\textit{PexelsEvents}) to high-speed (\textit{GameEvents}).
\begin{table*}[htb]
\centering
\small
\setlength{\tabcolsep}{4pt}
\begin{tabular}{c|c|cc|ccc|c}
\hline
Dataset & Video Clips & EC & ED & TE & TD & TL & Overlap Prob. \\
\hline
\textit{PexelsEvents} & 253,903 &  4.72 & 3.67 &1,197,973 & 4,391,589 &  164,639,876 & 68.0\%\\
\textit{GameEvents} & 79,959 & 16.01 & 1.24 & 1,280,208 & 1,584,762 & 42,368,009 & 99.63\%\\
\textit{RoboticsEvents} & 85,956 & 14.47 & 2.79 & 1,244,058 & 3,472,766 & 86,717,027 & 99.99\%\\
\hline
\end{tabular}

\caption{
Statistics of \omnievent. \textbf{EC} and \textbf{ED} denote the average number of events per sample and the average duration per event. \textbf{TE}, \textbf{TD}, and \textbf{TL} denote the total number of events, the total duration of all events, and the total character length of all text prompts, respectively. \textbf{Overlap Prob.} denotes the probability of event overlap.
}
\label{tab:data_stat}
\end{table*}
\section{Human Evaluation}\label{sec:human_eval}
\paragraph{Human annotation protocol.}
All temporal metrics are computed from human-verified event annotations rather than automatic detectors. 
We randomly sample 100 prompts from the evaluation dataset. 
For each prompt, we generate one video using the finetuned baseline and one video using our method with TIE, resulting in 200 generated videos in total. 
Each prompt contains a set of requested events with specified target intervals. 
For every generated video, human annotators are shown the video and the corresponding prompt event list. 
For each requested event, they are asked to answer three questions: 
(1) whether the event occurs in the video; 
(2) how much the observed event start time deviates from the requested start time; and 
(3) how much the observed event end time deviates from the requested end time. 
We employ 10 human annotators and compute the final metrics from the average of their responses. Figure~\ref{fig:ann_ui} gives the annotation user interface.
\begin{figure*}[!t]
    \centering
    \includegraphics[width=1.0\linewidth]{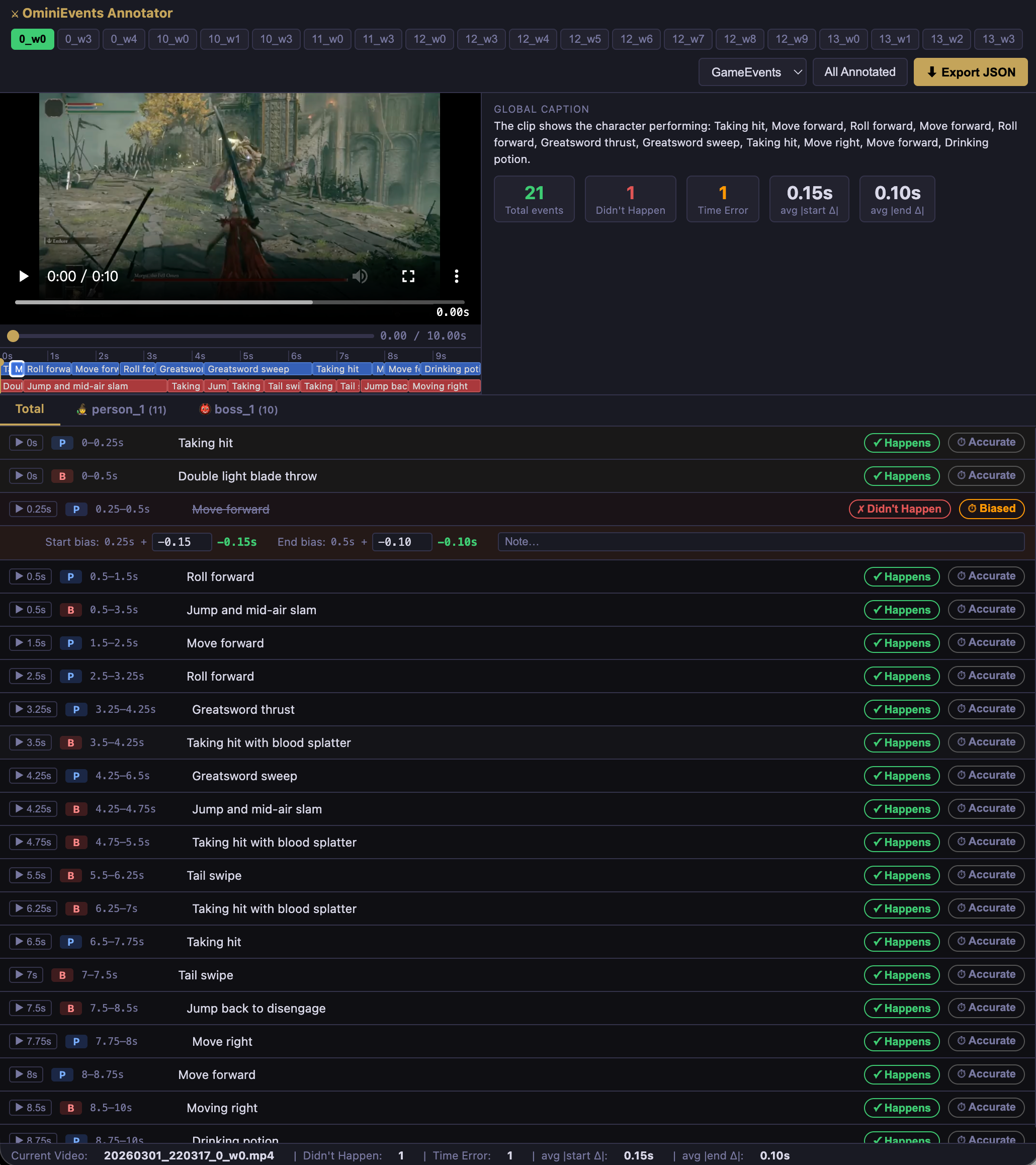}
    \captionsetup{justification=justified, singlelinecheck=true, skip=4pt}
    \caption{An example of the user interface for annotating temporal alignment.}
    \label{fig:ann_ui}
\end{figure*}

\paragraph{Event representation.}
Each prompt contains a set of requested events 
$\mathcal{E}=\{e_i\}_{i=1}^{N}$.
Each event $e_i$ is specified by an event description, an actor, and a target time interval $[s_i,t_i]$.
For each generated video, annotators provide a binary occurrence label $o_i\in\{0,1\}$ indicating whether $e_i$ appears in the video. 
If $o_i=1$, annotators additionally provide the start-time bias $b^s_i$ and end-time bias $b^t_i$ with respect to the requested interval:
\[
b^s_i = \hat{s}_i - s_i, \quad b^t_i = \hat{t}_i - t_i,
\]
where $[\hat{s}_i,\hat{t}_i]$ is the observed event interval in the generated video. 
For each event, we average $o_i$, $b^s_i$, and $b^t_i$ over the 10 annotators before computing the final metrics.

\paragraph{Event Occurrence.}
Event Occurrence measures the probability that a requested event is realized in the generated video:
\[
\mathrm{Occ} = \frac{1}{N}\sum_{i=1}^{N} o_i.
\]
This metric captures the most basic requirement of temporal control: the model must first generate the requested event before its timing can be evaluated. 
We compute it over all requested prompt events, so missing events directly reduce the score.

\paragraph{Temporal Error.}
Temporal Error measures the temporal deviation between the requested interval and the human-observed interval. 
If an event occurs, we use the average absolute boundary deviation:
\[
\mathrm{TE}_i = \frac{|b^s_i| + |b^t_i|}{2}.
\]
If an event is missing, its boundary deviation is undefined. 
We therefore assign the requested event duration as a missing-event penalty:
\[
\mathrm{TE}_i = t_i - s_i, \quad \text{if } o_i=0.
\]
The final Temporal Error is averaged over all requested events:
\[
\mathrm{TE} = \frac{1}{N}\sum_{i=1}^{N}\mathrm{TE}_i.
\]
This prevents a model from receiving an artificially low timing error by simply omitting difficult events. 
The event duration provides a natural penalty scale: missing a short event is less severe than missing a long event, while still being counted as a temporal failure.

\paragraph{Order Accuracy.}
Order Accuracy measures whether the generated video preserves the before/after relations specified by the prompt. 
For each pair of requested events $(e_i,e_j)$ in the same prompt, we compare their target center times:
\[
c_i=\frac{s_i+t_i}{2}, \quad c_j=\frac{s_j+t_j}{2}.
\]
If $c_i<c_j$, then $e_i$ is expected to occur before $e_j$. 
For occurred events, we compute the observed centers using the annotated biases:
\[
\hat{c}_i=\frac{(s_i+b^s_i)+(t_i+b^t_i)}{2}, \quad
\hat{c}_j=\frac{(s_j+b^s_j)+(t_j+b^t_j)}{2}.
\]
The pair is counted as correct only if both events occur and the observed ordering matches the requested ordering:
\[
\mathbf{1}[o_i=1,o_j=1]
\cdot
\mathbf{1}[
\mathrm{order}(\hat{c}_i,\hat{c}_j)=
\mathrm{order}(c_i,c_j)
].
\]
If either event is missing, the pair is counted as incorrect. 
The final Order Accuracy is the fraction of correctly ordered event pairs:
\[
\mathrm{OrderAcc}=
\frac{
\sum_{i<j}
\mathbf{1}[o_i=1,o_j=1]
\cdot
\mathbf{1}[
\mathrm{order}(\hat{c}_i,\hat{c}_j)=
\mathrm{order}(c_i,c_j)
]
}{
\binom{N}{2}
}.
\]
This metric evaluates global temporal structure. 
A video may contain many requested events but still fail the prompt if their sequence is incorrect.

\paragraph{Overlap Accuracy.}
Overlap Accuracy measures whether events that are specified to overlap or occur concurrently in the prompt remain overlapping in the generated video. 
We first identify event pairs that overlap in the prompt:
\[
\mathcal{O}=
\{(i,j):\min(t_i,t_j)>\max(s_i,s_j)\}.
\]
For each pair $(i,j)\in\mathcal{O}$, we compute the observed intervals from the human-annotated biases:
\[
\hat{s}_i=s_i+b^s_i,\quad \hat{t}_i=t_i+b^t_i.
\]
The pair is counted as correct only if both events occur and their observed intervals overlap:
\[
\mathbf{1}[o_i=1,o_j=1]
\cdot
\mathbf{1}[
\min(\hat{t}_i,\hat{t}_j)>
\max(\hat{s}_i,\hat{s}_j)
].
\]
The final Overlap Accuracy is:
\[
\mathrm{OverlapAcc}=
\frac{1}{|\mathcal{O}|}
\sum_{(i,j)\in\mathcal{O}}
\mathbf{1}[o_i=1,o_j=1]
\cdot
\mathbf{1}[
\min(\hat{t}_i,\hat{t}_j)>
\max(\hat{s}_i,\hat{s}_j)
].
\]
If either event in an overlapping pair is missing, the pair is counted as incorrect. 
This metric is complementary to Order Accuracy: Order Accuracy evaluates before/after relations, while Overlap Accuracy evaluates concurrency preservation.

\paragraph{Temporal Constraint Satisfaction Rate.}
Temporal Constraint Satisfaction Rate (TCSR) summarizes whether temporal constraints are satisfied. 
This metric evaluates complete prompt-level temporal faithfulness as:
\[
\mathrm{TCSR}=
\frac{1}{|\mathcal{C}|}
\sum_{c\in\mathcal{C}}
\mathbf{1}[c\ \mathrm{is\ satisfied}].
\]
A constraint is satisfied only when the required event or event relation is realized and its temporal condition is met. 
For event interval constraints, this requires the event to occur and its start/end deviations to fall within the annotation tolerance, that is
 \begin{itemize}
     \item it is the start time of an interval constraint and the error from the ground truth start time is within 0.25s;
    \item or it is the end time of an interval constraint and the error from the ground truth end time is within 0.25s.
 \end{itemize}
For order constraints, both events must occur and their before/after relation must be preserved. 
For overlap constraints, both events must occur and their observed intervals must overlap. 

\paragraph{Rationale.}
These metrics are designed to jointly evaluate event realization, interval alignment, temporal ordering, concurrency, and full prompt-level constraint satisfaction. 
A key design choice is that missing events are explicitly penalized. 
Missing events reduce Event Occurrence, receive a duration-based penalty in Temporal Error, invalidate any related order or overlap pairs, and violate the corresponding TCSR constraints. 
This avoids a degenerate evaluation where a model appears temporally accurate only because it omits difficult events. 
By relying on human annotations of event occurrence and temporal deviations, the evaluation directly measures whether generated videos satisfy the intended prompt-level temporal structure.

\section{More Results}
\label{append:moreresults}
\subsection{Results on StoryBench}
For StoryBench, we rescale event timelines to 10 seconds and use the background description as a persistent event spanning the full video. \TIE outperforms previous multi-event video generation method on FVD, CLIP-Event, and VideoScore dimensions, indicating stronger temporal organization and visual fidelity.
\begin{table}[htb]
\centering
\small
\setlength{\tabcolsep}{4pt}
\begin{tabular}{c|cc|c|cccc}
\hline
Method & FID$\downarrow$ & FVD$\downarrow$ & CLIP-Event $\uparrow$ & VQ$\uparrow$ & TC$\uparrow$ & DD$\uparrow$ & TA$\uparrow$\\ \hline
MinT & 40.87 & 484.44 & 0.270 & 2.56 & 2.44 & 3.32 & 2.92\\
{\lmrope} & \textbf{23.58} & \textbf{236.45} & \textbf{0.290} & \textbf{3.39} & \textbf{3.21} & \textbf{3.57} & \textbf{3.14}\\ \hline
\end{tabular}
\caption{Text-to-Video Generation Result on StoryBench: The evaluation results of MinT come from their paper since they are not open-sourced.}
\label{tab:expv3_storybench}
\end{table}
\subsection{\TIE as Prompt Extender}
A key advantage of \TIE is its capacity for structured prompt extension. We use LLMs (e.g., Qwen3-Max) to expand simple VBench \cite{huang2023vbench} prompts into our structured event representation.
\begin{table*}[htb]
\centering
\footnotesize
\setlength{\tabcolsep}{2pt}
\begin{tabular}{c|ccc|cc|cc|cccccc}
\hline
method  & \tiny\begin{tabular}[c]{@{}l@{}}Dynamic\\Degree\end{tabular}  & \tiny\begin{tabular}[c]{@{}l@{}} Imaging\\ Quality\end{tabular}& \tiny\begin{tabular}[c]{@{}l@{}}Aesthetic\\ Quality\end{tabular}& \tiny\begin{tabular}[c]{@{}l@{}}Motion\\ Smooth.\end{tabular} &\tiny\begin{tabular}[c]{@{}l@{}}Human\\ Action\end{tabular} & \tiny\begin{tabular}[c]{@{}l@{}}Overall\\ Consist.\end{tabular} & \tiny\begin{tabular}[c]{@{}l@{}}Subject\\ Consist.\end{tabular}& \tiny Color &\tiny\begin{tabular}[c]{@{}l@{}}Object\\ Class\end{tabular}&\tiny\begin{tabular}[c]{@{}l@{}}Multiple\\ Objects\end{tabular}&\tiny\begin{tabular}[c]{@{}l@{}}Spatial\\ Relation.\end{tabular}&\tiny\begin{tabular}[c]{@{}l@{}}Tempo.\\ Style\end{tabular}&\tiny\begin{tabular}[c]{@{}l@{}}Appear.\\ Style\end{tabular}\\ \hline
Baseline                               &   0.375             &        \textbf{0.665}                                                  &      0.607                                                          &      0.985                                                  &                                    0.72 & 0.237                       &         0.938          &   \textbf{0.895} & 0.747 & 0.562 & 0.743 & 0.231 & 0.209                                        \\
WanExtender                     &      0.119              &           0.643      & \textbf{0.658}     & \textbf{0.993} & 0.962 & 0.266 & \textbf{0.948} & 0.830 & \textbf{0.925} & \textbf{0.797} & 0.780 & 0.243 & 0.219                          \\
{\lmrope} &  \textbf{0.763} & 0.623 & 0.599 & 0.979 &  \textbf{0.994} & \textbf{0.279} & 0.924 & 0.892 & 0.899 & 0.741 & \textbf{0.804} &  \textbf{0.254} & \textbf{0.247} \\ \hline
\end{tabular}
\caption{Prompt Extending Results: These metrics are separated into four blocks related to Visual Quality, Motion, Consistency and Prompt Faithfulness. We find that WanExtender increases visual quality at the expense of dynamic degree, often producing nearly static scenes. In contrast, the \TIE-enhanced variant significantly increases dynamic degree, retains comparable visual quality, and improves prompt faithfulness.}
\label{tab:prompt_extend_results}
\vspace{-1.0em}
\end{table*}
Compared to the official \textbf{WanExtender}, the \TIE-enhanced variant shows superior fidelity to long-horizon descriptions. As shown in \cref{tab:prompt_extend_results} and \cref{fig:prompt_extend_case}, it generates more dynamic content that faithfully follows detailed spatial and temporal relationships, indicating that \TIE provides a robust structured extension mechanism for complex video synthesis.

\begin{figure*}[htb]
    \centering
    \hspace{30pt}
    \includegraphics[width=1.0\linewidth]{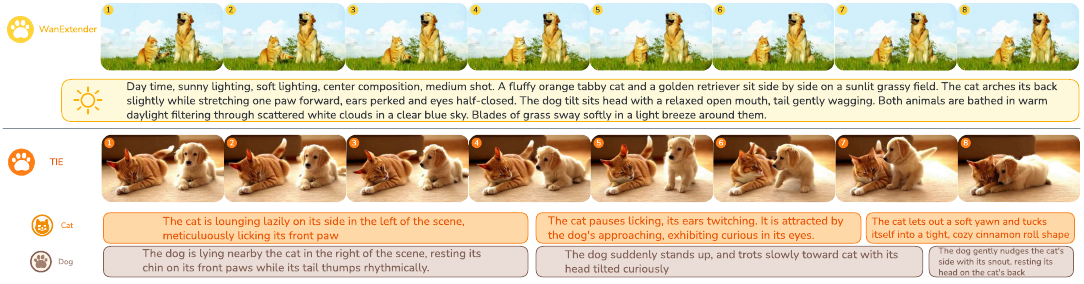}
    \captionsetup{justification=justified, singlelinecheck=true}
    \caption{Comparison Results between WanExtender and the \TIE-enhanced variant: the original short prompt from VBench is ``a cat and a dog''. VLM extends the short prompt into a series of consistent behaviors which result in an interaction between a dog and a cat.}
    \label{fig:prompt_extend_case}
    \vspace{-1.5em}
\end{figure*}
\subsection{Case Study on the RoboticsEvents}\label{sec:robo_case}
In Figure \ref{fig:robotics}, we show a case for precise temporal control of TIE. We adjust the time interval of one event of the right arm. The right arm should reach out to pull the door of the oven. We specify this accurate start time from $2.0s$ to $5.0s$, TIE accurately performs the requirement, while the Finetuned one does not respond in time.
\begin{figure*}[htb]
    \centering
    \hspace{30pt}
    \includegraphics[width=1.0\linewidth]{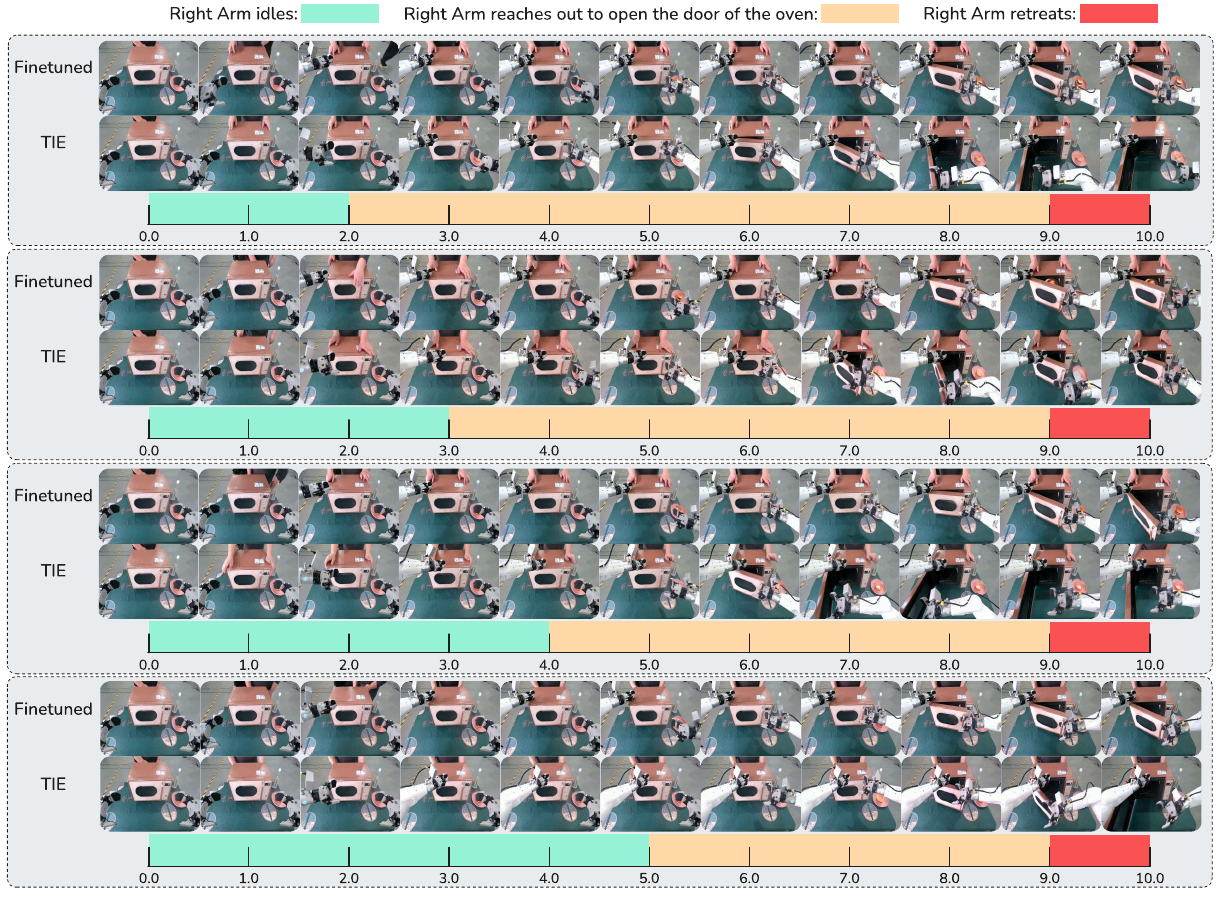}
    \captionsetup{justification=justified, singlelinecheck=true}
    \caption{Generation results on \textit{RoboticsEvents}. We compare the temporal response behavior of the Finetuned baseline and \TIE.}
    
    \label{fig:robotics}
\end{figure*}
\subsection{Boundary Experiments}
\paragraph{Performance for different numbers of entities and events} Evaluation results with different amount of descriptions split can be found in \cref{tab:descsplit}, and results with different amount of entities split can be found in \cref{tab:entsplit}. We can see that with more events and entities, the temporal alignment of \TIE consistently drops, while more descriptions provide more information and increase the video quality. 

\begin{table}[htb]
\centering
\small
\setlength{\tabcolsep}{4pt}
\begin{tabular}{c|c|ccccc}
\hline
Description Count & CLIP-Event $\uparrow$ & VQ $\uparrow$ & TC $\uparrow$  & DD $\uparrow$ & TA $\uparrow$ & FC $\uparrow$\\ \hline
\lbrack 3,4\rbrack &	0.2846& 	2.977 &	2.941 &	2.803 &	2.741 &	2.853\\
\lbrack5,6,7\rbrack &	0.2694& 	2.975 &	2.920& 	2.907 &	2.768 &	2.851\\
\lbrack8,9,10\rbrack& 	0.2394 &	2.965 &	2.900 &	2.943 &	2.822 &	2.837\\
\lbrack11,12,13\rbrack &	0.2267 &	3.057 &	2.987 &	3.072 &	2.948 &	2.934\\
\lbrack14,15,16\rbrack& 	0.2138 &	3.293& 	3.285 &	3.290 &	3.093 &	3.220\\ \hline
\end{tabular}
\caption{\TIE Performance on various description count splits.}
\label{tab:descsplit}
\end{table}
\begin{table}[htb]
\centering
\small
\setlength{\tabcolsep}{4pt}
\begin{tabular}{c|c|ccccc}
\hline
Entity Count & CLIP-Event $\uparrow$ & VQ $\uparrow$ & TC $\uparrow$  & DD $\uparrow$ & TA $\uparrow$ & FC $\uparrow$\\ \hline
1 &	0.2801& 	2.942 &	2.887 	&2.829 	&2.711 &	2.811\\
2 &	0.2400 &	2.971 &	2.897 &	2.949 &	2.837 &	2.821\\
3 &	0.2221 &	3.238 &	3.212 &	3.209 &	3.047 &	3.161\\
4 &	0.2140 &	3.176 &	3.157 &	3.140 	&3.109& 	3.120\\
5 &	0.2150 &	3.310 &	3.338 &	3.299 &	3.143 &	3.289\\
6 &	0.1977 &	3.299 &	3.308 &	3.371 &	3.176 &	3.241\\
7 &	0.2089& 	3.501 &	3.542 &	3.658 &	3.223 	&3.503\\ \hline
\end{tabular}
\caption{\TIE Performance on various entity count splits.}
\label{tab:entsplit}
\end{table}
\subsection{Experiment Settings}
During all experiments, we use 32 H800 GPUs for training. All 720p models (on Pexels-250k Dataset) are trained in batch size $512$ for 6,500 steps with a learning rate of $10^{-5}$ and an AdamW \cite{Kingma2014AdamAM,Loshchilov2017FixingWD} optimizer. All inferences are performed with $50$ steps, with a CFG scale of $5.0$. All 480p models (on domain-specific datasets) are trained in batch size $128$ for 20,000 steps with a learning rate of $3\times 10^{-5}$.

\section{Discussions}
\subsection{Limitations}
Modern VLMs still may not be capable enough to generate such a heavy prompt. Since captioning events requires understanding high-FPS videos and long contexts. VLM could generally output hallucinated contents, mismatched entities and behaviors, inconsistent entities and many other kinds of factual failures. Reflective Agent can somehow refine these problems but for modern VLMs that require API calls, retrying for too many times could be expensive.
\subsection{Future Work}
We believe \TIE can be naturally integrated with long-video generation paradigms, such as next-chunk generation with KV-Cache, resulting in the next-level AIGC visual content creation suite. And also there could be natural applications for data generation / enrichment like the embodiment-AI scenario and digital human. We are on the way to explore the new frontier of interactive, customized and content-rich visual immersion / creation.
}{}
\iftoggle{isSupplementary}{
    \section{Supplementaries}
}{}
\end{document}